\documentclass{article}
\pdfoutput=1
\usepackage{etoolbox, tikz}

\usepackage{booktabs}       

\newtoggle{arxivupload}
\togglefalse{arxivupload}

\usepackage{boxedminipage}
\usepackage{multirow,nicefrac}
\usepackage{makecell,upgreek}
\usepackage{footnote}
\usepackage{longtable}
\usepackage[shortlabels]{enumitem}
\usepackage{tablefootnote}
\usepackage[T1]{fontenc}
\usepackage{verbatim} 
\usepackage[utf8]{inputenc}

\usepackage{fullpage}

\usepackage[breaklinks=true,linkcolor=black]{hyperref}

\usepackage{amsthm}

\usepackage{amsfonts,amssymb,mathrsfs}
\usepackage{mathtools}
\DeclareMathSymbol{\shortminus}{\mathbin}{AMSa}{"39}

	\usepackage{natbib}

\usepackage{prettyref}

\usepackage[vlined,boxed]{algorithm2e}

\usepackage[capitalise,nameinlink]{cleveref}
\Crefname{equation}{Eq.}{Eqs.}
\Crefname{assumption}{Assumption}{Assumptions}
\Crefname{condition}{Condition}{Conditions}

\usepackage{breakcites,bm}

\hypersetup{colorlinks,citecolor=blue,linkcolor = blue}
\usepackage{latexsym}
\usepackage{relsize}
\usepackage{thm-restate}
\usepackage{appendix}

\usepackage{xcolor}
\usepackage{dsfont}

\newcommand{\defeq}{:=}

\numberwithin{equation}{section}

\newcommand{\calM}{\mathcal{M}}

\DeclareFontFamily{U}{mathx}{\hyphenchar\font45}
\DeclareFontShape{U}{mathx}{m}{n}{
      <5> <6> <7> <8> <9> <10>
      <10.95> <12> <14.4> <17.28> <20.74> <24.88>
      mathx10
      }{}
\DeclareSymbolFont{mathx}{U}{mathx}{m}{n}
\DeclareFontSubstitution{U}{mathx}{m}{n}
\DeclareMathAccent{\widecheck}{0}{mathx}{"71}
\DeclareMathAccent{\wideparen}{0}{mathx}{"75}

\newcommand{\ignore}[1]{}

\newcommand{\twonorm}[1]{\|#1\|_2}

\newcommand{\R}{\mathbb{R}}

\newcommand{\opnorm}[1]{\|#1\|_{\op}}
\newcommand{\infnorm}[1]{\|#1\|_{\infty}}

\newcommand{\calN}{\mathcal{N}}

\newcommand{\op}{\mathrm{op}}

\newcommand{\calF}{\mathcal{F}}

\newcommand{\calD}{\mathcal{D}}

	\theoremstyle{plain}
	\newtheorem{theorem}{Theorem}

	\newtheorem{lemma}{Lemma}[section]

	\newtheorem{claim}[lemma]{Claim}
	\newtheorem{fact}[lemma]{Fact}
	\newtheorem{corollary}{Corollary}[section]
	\newtheorem{proposition}[lemma]{Proposition}

	\theoremstyle{definition}

	\newtheorem{definition}{Definition}[section]

  \newtheorem{assumption}{Assumption}
  \newtheorem{condition}{Condition}[section]

\makeatletter
\newcommand{\neutralize}[1]{\expandafter\let\csname c@#1\endcsname\count@}
\makeatother

\newtheorem*{theorem*}{Theorem}
\newtheorem*{lemma*}{Lemma}
\newtheorem*{corollary*}{Corollary}
\newtheorem*{proposition*}{Proposition}
\newtheorem*{claim*}{Claim}
\newtheorem*{fact*}{Fact}
\newtheorem*{observation*}{Observation}

\newtheorem*{definition*}{Definition}
\newtheorem*{remark*}{Remark}
\newtheorem*{example*}{Example}

\newtheoremstyle{named}{}{}{\itshape}{}{\bfseries}{}{.5em}{\Cref{#3} {\normalfont (informal)} }
{}
\theoremstyle{named}

\theoremstyle{plain}

\DeclareMathAlphabet{\mathbfsf}{\encodingdefault}{\sfdefault}{bx}{n}

\DeclareMathOperator*{\argmin}{arg\,min}

\DeclareMathOperator*{\supp}{supp}

\newcommand{\E}{\mathbb{E}}

\newcommand{\trace}{\mathrm{tr}}

\newcommand{\eps}{\varepsilon}

\renewcommand{\leq}{~\le~}
\renewcommand{\geq}{~\ge~}

\let\oldtfrac\tfrac
\renewcommand{\tfrac}[2]{\smash{\oldtfrac{#1}{#2}}}

\let\nablaold\nabla
\renewcommand{\nabla}{\nablaold\mkern-2.5mu}

\newcommand{\diag}{\mathrm{diag}}

\newcommand{\cond}{\mathsf{cond}}
\newcommand{\trop}{P^{(\pi)}}

\newcommand{\epsinf}{\epsilon_{\infty}}

\newcommand{\levcov}{\rho_{s}}
\newcommand{\levnext}{\rho_{s'}}

\newcommand{\thetahatls}{\widehat{\theta}_{\mathrm{LS}}}

\newcommand{\crossvar}{\sigma^2_{\mathrm{cr}}}

\newcommand{\epsop}{\epsilon_{\mathrm{op}}}
\newcommand{\epsz}{\epsilon_{r}}

\newcommand{\tsg}{\theta^\star_\gamma}
\newcommand{\cD}{\mathcal{D}}
\newcommand{\covhat}{\widehat{\Sigma}_{\mathrm{cov}}}
\newcommand{\cov}{\Sigma_{\mathrm{cov}}}
\newcommand{\cross}{\Sigma_{\mathrm{cr}}}
\newcommand{\crosshat}{\widehat{\Sigma}_{\mathrm{cr}}}
\newcommand{\smin}{\sigma_{\min}}
\newcommand{\Ahat}{\widehat{A}}
\newcommand{\cds}{\mathcal{C}_{\mathrm{ds}}}
\newcommand{\Cds}{\cds}
\newcommand{\gammads}{\gamma_{\mathrm{ds}}}

\newcommand{\traceb}[1]{\trace\left[ #1 \right]}

\renewcommand{\epsilon}{\varepsilon}

\newcommand{\dlyaps}[1]{\mathsf{dlyap}(#1)}

\newcommand{\fnorm}[1]{\| #1 \|_{\mathrm{F}}}

\newcommand{\nextcov}{\Sigma_{\mathrm{next}}}

\newcommand{\states}{\mathcal{S}}
\newcommand{\actions}{\mathcal{A}}
\newcommand{\tpr}{\theta_{\phi, r}}
\newcommand{\tphi}{\widetilde{\phi}}
\newcommand{\wphi}{\phi_w}
\newcommand{\tcov}{\widetilde{\Sigma}_{\mathrm{cov}}}
\newcommand{\tcross}{\widetilde{\Sigma}_{\mathrm{cr}}}
\newcommand{\varcov}{\sigma^2_{\mathrm{cov}}}
\newcommand{\whitecross}{\gamma \cov^{-1/2}\cross \cov^{-1/2}}

\newcommand{\Qhat}{\widehat{Q}^\pi}
\newcommand{\Qpi}{Q^{\pi}}

\newcommand{\altm}{\overline{\calM}}
\newcommand{\altreward}{\overline{R}}
\newcommand{\altqpi}{\bar{Q}^\pi}
\newcommand{\altcov}{\bar{\Sigma}_{\mathrm{cov}}}
\newcommand{\altnext}{\bar{\Sigma}_{\mathrm{next}}}
\newcommand{\altcross}{\bar{\Sigma}_{\mathrm{cr}}}
\newcommand{\altthetar}{\bar{\theta}_{\phi,r}}
\newcommand{\tbrm}{\theta_{\mathrm{BRM}}}
\newcommand{\tsi}{\theta^\star_\infty}
\newcommand{\thetahat}{\widehat{\theta}}
\newcommand{\epsfp}{\eps_{\mathrm{fp}}}
\newcommand{\tfp}{\theta^\star_{\mathrm{fp}}}
\newcommand{\rewardvar}{\sigma^2_{r}}

\title{A Complete Characterization of Linear Estimators\\ for Offline Policy Evaluation}
\author{Juan C. Perdomo \\\texttt{jcperdomo@berkeley.edu}\\ University of California, Berkeley \and Akshay Krishnamurthy\\\texttt{akshaykr@microsoft.com}\\ Microsoft Research \and Peter Bartlett \\\texttt{peter@berkeley.edu} \\  University of California, Berkeley \and Sham Kakade  \\\texttt{sham@seas.harvard.edu} \\ Harvard University
}

\date{\today}

\begin{document}

\maketitle
\begin{abstract}
Offline policy evaluation is a fundamental statistical problem in reinforcement learning that involves estimating the value function of some decision-making policy given data collected by a potentially different policy. 
In order to tackle problems with complex, high-dimensional observations, there has been significant interest from theoreticians and practitioners alike in understanding the possibility of \emph{function approximation} in reinforcement learning. 
Despite significant study, a sharp characterization of when we might expect offline policy evaluation to be tractable, even in the simplest setting of \emph{linear} function approximation, has so far remained elusive, with a surprising number of strong negative results recently appearing in the literature. 

In this work, we identify simple control-theoretic and 
linear-algebraic conditions that are necessary and sufficient for classical methods, in particular Fitted Q-iteration (FQI) and least squares temporal difference learning (LSTD), to succeed at offline policy evaluation. 
Using this characterization, we establish a precise hierarchy of regimes under which these estimators succeed. 
We prove that LSTD works under strictly weaker conditions than FQI. Furthermore, we establish that if a problem is not solvable via LSTD, then it cannot be solved by a broad class of linear estimators, even in the limit of infinite data. 
Taken together, our results provide a complete picture of the
   behavior of linear estimators for offline policy evaluation, unify previously disparate analyses of canonical algorithms, and provide significantly sharper notions of the underlying statistical complexity of offline policy evaluation. 

\end{abstract}

\section{Introduction}

A central component of a practical sequential decision making
  system is its ability to cope with high-dimensional and complex data
  sources. While feature engineering or discretization techniques
  can in principle be used to address the challenges associated with
  complex data, these approaches require significant domain expertise
  and suffer from a curse-of-dimensionality phenomenon  that limit their
  practical relevance. Instead, the use of more general \emph{function
  approximation methods} for reinforcement learning (RL) promises to avoid
  these drawbacks. Consequently, understanding these methods has long
  been a topic of interest to theoreticians and practitioners alike.

 While the use of nonlinear methods is by now common in the
  empirical reinforcement learning literature, the much simpler linear
  function approximation setting remains somewhat poorly understood
  theoretically, 
  despite decades of study. Indeed, recently
  there has been a surge of research effort focusing on necessary and
  sufficient conditions for reinforcement learning with linear
  function approximation, including the first provably efficient
  algorithms for online exploration~\citep{yang2020reinforcement,jin2020provably} and a number of surprising
  statistical lower bounds that hold even under strong
  assumptions~\citep{wang2021exponential,weisz2021exponential,weisz2021tensorplan}. This line of work represents substantial progress, yet
  we still lack a clear picture as to precisely when and why RL with linear
  function approximation is tractable.  

As a step towards providing this clarity, in this paper we
  focus on the simpler \emph{offline policy evaluation} problem
  (OPE) in infinite horizon, discounted MDPs, under the assumption that the action-value function is
  \emph{linearly realizable} in some known features. Here, rather than
  interacting with an environment to maximize reward as in the standard
  RL formulation, the goal is to estimate the performance of a given
  decision-making policy by leveraging an observational dataset
  collected by a potentially different policy.  OPE is perhaps the
  simplest, non-trivial setting in which to study function
  approximation in RL. 
  It is also
  practically relevant in its own right: both OPE and the closely-related
  offline policy \emph{optimization} problem represent a promising avenue toward applying
  RL in safety-critical domains where active exploration is infeasible. Moreover, the principles developed for OPE
  are routinely used in online RL algorithms.

Fitted Q-iteration (FQI) \citep{ernst2005tree, riedmiller2005neural} and least squares temporal difference learning (LSTD) \citep{bradtke1996linear, boyan1999least, nedic2003least} are canonical algorithms for offline policy evaluation with function approximation. 
These simple, moment-based methods are some of the most popular approaches in practice and have served as inspiration for recent empirical breakthroughs in RL \citep{mnih2015human}.
They have also been the subject of intense theoretical investigation,
with early results on convergence and instability described
by~\citet{bertsekas1995neuro,tsitsiklis1996feature} as well as several
more recent results~\citep{antos2008learning,chen2019information,lazaric2012finite}.
Nevertheless, a sharp finite sample characterization of the behavior of FQI and LSTD, even in the linear realizability setting, remains undeveloped.

In this paper, we identify
necessary and sufficient conditions for FQI and LSTD to succeed at offline policy evaluation under linear realizability. 
In doing so, we establish a precise hierarchy of conditions under which these methods work; in particular, we prove that LSTD succeeds under strictly weaker assumptions than FQI. 
Moreover, if an offline policy evaluation problem is not solvable via LSTD, then it cannot be solved by any linear, moment-based method (see \Cref{def:linear_estimator}) even in the limit of infinite data. 
Our characterization draws upon ideas from the theory of Lyapunov stability and  provides a new, unifying perspective on the statistical complexity of offline policy evaluation. 
In particular, we show how traditional quantities, such as the ``effective horizon'', fail to capture the true complexity of the problem (\Cref{subsec:sufficiency,subsec:lspe_upper}) and 
propose instance-dependent measures which are significantly sharper. Furthermore, our results unify previously disparate analyses for FQI and LSTD as our conditions are implied by prior assumptions (\Cref{subsec:connections,subsec:invertibility_connections}). Taken together, our results provide a complete picture of the
possibilities and limitations of linear estimators for offline policy evaluation under linear realizability.


\subsection{Linear estimators \& the offline policy evaluation problem}
\label{subsec:prob_def}
Let $\calM \defeq (\states, \actions, P, R, \gamma)$ denote an infinite horizon, $\gamma$-discounted MDP where $\states$ is the set of states, $\actions$ is the set of actions, $R:\states \times \actions \rightarrow \Delta([-1,1])$ is the random reward function, and $P:\states \times \actions \rightarrow \Delta(\states)$ is the transition operator, which defines a distribution over states for every pair $(s,a)$. 
 The action-value function $Q^\pi$ captures the expected total reward achieved by a randomized policy $\pi: \states \to \Delta(\actions)$ from an initial state-action pair $(s,a)$ when the trajectory is generated such that for each time step $h$, $a_h \sim \pi(s_h)$ and $s_{h+1} \sim P(\cdot \mid s_h,a_h)$.
\begin{align}
\label{eq:q_function}
Q^\pi(s, a) \defeq \E\left[ \sum_{h=0}^\infty \gamma^h r(s_h, a_h) \mid (s_0, a_0)=(s,a), \pi\right].
\end{align}
In the offline policy evaluation problem, we are given a policy $\pi$ and a dataset $\{(s_i,a_i,r_i(s_i,a_i), s_i', a_i')\}_{i=1}^n$ of observed transitions and rewards, where the initial pair $(s_i,a_i)$ is sampled from some \emph{arbitrary} distribution $\calD$, $r_i(s_i,a_i) \sim R(s_i,a_i)$, the next state is sampled from the transition operator $s_i' \sim P(\cdot \mid s_i,a_i)$, and the next action $a_i' \sim \pi(s_i')$ is sampled according to $\pi$.\footnote{We ``augment'' the dataset to include the next state action $a' \sim \pi(s')$ purely for notational convenience.}
Our goal is to return an estimate $\Qhat$ of $\Qpi$. For concreteness, we measure performance via $\E_{(s,a)\sim\calD}{|\Qhat(s,a) - \Qpi(s,a)|}$ and we ask that this quantity is vanishingly small with high probability over the draw of the dataset. 
For simplicity, we assume that samples are drawn i.i.d. via the procedure described above.\footnote{In particular, extensions to Markovian data, where samples are drawn from an ergodic chain, are fairly well-understood, see e.g., \cite{mou2021optimal, nagaraj2020least}. Overall, the statistical rates in the Markovian setting mimic those obtained under i.i.d assumptions, up to mixing time factors.}

As we would like to develop methods that scale to settings where the cardinalities of the sets $\states$ and $\actions$ are large or infinite, our focus is on understanding  policy evaluation using linear function approximation, as per the following definition:

\begin{assumption}[Linear Realizability]
\label{ass:realizability}
$Q^\pi$  is linearly realizable\footnote{Note that realizability of $\Qpi$ does not imply that the rewards are linearly realizable. We say that rewards are linearly realizable in a feature mapping $\phi:\states \times 
\actions \rightarrow \R^d$ if there exists $\theta^\star_r \in \R^d$ such that for all $(s,a) \in \states \times \actions$, $\phi(s,a)^\top \theta^*_r = \E r(s,a)$.} in a known feature map $\phi: \states \times \actions \rightarrow \R^d$ if there exists a vector $\theta_\gamma^\star \in \R^d$ such that for all $(s,a) \in \states \times \actions$,
$Q^\pi(s,a) = \phi(s,a)^\top \theta_\gamma^\star$.	
\end{assumption}

\paragraph{Fitted Q-iteration.} As mentioned previously, fitted Q-iteration is one of the most popular algorithms for policy evaluation in practice and can in principle work with any function approximation method. In the linear case, given a dataset $\{(s_i, a_i, r_i(s_i,a_i), s_i', a_i')\}_{i=1}^n$ and an initial vector $\widehat{\theta}_0$, FQI iteratively solves least squares regression problems of the form
\begin{align}
\label{eq:fqi_def}
	\widehat{\theta}_{t+1} \in \argmin_{\theta} \sum_{i=1}^{n} \left(\phi(s_i,a_i)^\top \theta - r(s_i,a_i) - \gamma \phi(s_i', a_i')^\top\widehat{\theta}_t \right)^2,
\end{align} 
for some number of rounds $T$ and returns the estimator $\Qhat(s,a) := \phi(s,a)^\top \widehat{\theta}_T$.

\paragraph{Least squares temporal difference learning.} In the linear function approximation setting, the vector $\theta_\gamma^\star$ which realizes $\Qpi$ in the feature mapping $\phi$ satisfies the fixed point equation,\footnote{This fixed point relationship comes from examining the definition of $\Qpi(s,a)$ which satisfies, $\Qpi(s,a) = \E r(s,a) + \gamma \cdot \E\Qpi(s', a')$ point-wise over $(s,a)$. The precise equation follows from substituting in $\Qpi = \phi(s,a)^\top \tsg$.}
\begin{align}
\label{eq:bellman_fixed_point}
	\cov \theta^\star_\gamma = \gamma \cross \theta^\star_\gamma + \theta_{\phi,r}.
\end{align}
Here, $\cov$ if the offline feature covariance matrix, $\cross$ is the cross-covariance matrix between time-adjacent features, and $\theta_{\phi,r}$ is the mean feature-reward vector. (see \Cref{eq:pop_defs,eq:thetaphir} for formal definitions).
LSTD tries to approximate $\theta_\gamma^\star$ by computing the plug-in estimate to the closed-form solution to the equation above, 
\begin{align}
\label{eq:lstd_def}
	\thetahatls \defeq (I - \gamma \covhat^{-1} \crosshat)^{\dagger} \covhat^{-1} \widehat{\theta}_{\phi,r} = (\covhat - \gamma \crosshat)^\dagger \widehat{\theta}_{\phi,r}.
\end{align}  
and returns $\Qhat(s,a) := \phi(s,a)^\top \thetahatls$ \citep{bradtke1996linear}. We focus on the unregularized variant of both of these algorithms. However, similar insights apply to the regularized cases (see \Cref{subsec:ridge}).

\subsection{Our contributions}
\label{subsec:contributions}
The main result of our work is that we identify simple linear algebraic conditions which exactly characterize when (and why) linear estimators will succeed at offline policy evaluation under linear realizability of $\Qpi$. 
Under these conditions, which we introduce below, we establish upper bounds on the sample complexity of offline policy evaluation which scale with: $(1)$ for FQI, the operator norm of the solution to a particular discrete-time Lyapunov equation, and $(2)$ for LSTD, the minimum singular value of an instance-dependent matrix. In both cases, we illustrate how our results unify previously disparate analyses of these algorithms, and demonstrate how our new instance-dependent quantities provide sharper notions of the statistical complexity of OPE when compared to bounds that explicitly depend on traditional parameters such as the ``effective horizon'', i.e., $1/(1-\gamma)$.

Our conditions can be introduced rather succinctly.
For FQI, the key definitions and assumptions are:
\begin{align}
\label{eq:pop_defs}
\cov \defeq \underset{(s,a) \sim \calD}{\E} \left [\phi(s,a) \phi(s,a)^\top \right], \quad \cross \defeq \underset{\substack{(s,a) \sim \calD\\ s' \sim P(\cdot \mid s,a), \; a' \sim \pi(s')}}{\E} \left [\phi(s,a) \phi(s',a')^\top \right].
\end{align}
\begin{assumption}[Stability]
\label{ass:stability}
The matrix $\cov$ is full rank and $\rho(\gamma \cov^{-1} \cross) < 1$.
\end{assumption}
Here, $\cov$ is the offline state-action covariance, $\cross$ is the cross-covariance, $\whitecross$ is the \emph{whitened cross-covariance},\footnote{For any matrix $A$ and invertible matrix $L$, the eigenvalues of A and $L^{-1}A L$ are identical. Therefore, one could equivalently state \Cref{ass:stability,ass:full_rank} in terms of $\whitecross$.} and  $\rho(A) = \max_i |\lambda_i(A)|$ is the \emph{spectral radius} of the matrix $A$.
The assumption that $\cov$ is full rank is not fundamental and is included primarily to simplify the presentation.\footnote{For example, the results carry over if all features $\phi(s,a)$ lie in a low dimensional subspace.}
If \Cref{ass:stability} holds, we let $P_\gamma$ be the unique solution (over $X$) to the Lyapunov equation, 
\begin{align*}
X = (\gamma \cov^{-1/2} \cross \cov^{-1/2})^\top X (\gamma \cov^{-1/2} \cross \cov^{-1/2}) + I.
\end{align*}
Our first main result is that, under stability, FQI satisfies the following error guarantee:

\newtheorem*{thm:main_result_informal}{\Cref{theorem:main_result} (Informal)}
\begin{thm:main_result_informal}  
Let $\Qhat(s,a) = \phi(s,a)^\top \widehat{\theta}_T$, where $\widehat{\theta}_T$ is the $T$-step FQI solution. Under \Cref{ass:realizability,ass:stability}, as well as standard regularity assumptions for linear regression, for $n$ large enough,
\begin{align*}
	\E_{\calD} |Q^\pi(s,a) - \widehat{Q}^\pi(s,a)| \lesssim \cond(P_\gamma) \opnorm{P_\gamma}^2  \sqrt{\frac{d\log(1/\delta)}{n}} + \mathcal{O}(\exp(-T)),
\end{align*}
with probability $1-\delta$. Here, $\cond(\cdot)$ and $\opnorm{\cdot}$ denote the condition number and operator norm.
\end{thm:main_result_informal}
For the sake of clarity, we have suppressed dependence on universal constants and other quantities which arise from standard analysis of linear regression in the informal statement of the upper bound. 
Since $P_\gamma \succeq I$, $\cond(P_\gamma)$ can always be crudely upper bounded by the operator norm, so that the primary factor, beyond the standard $\sqrt{d/n}$ term for linear regression, is the dependence on $\opnorm{P_\gamma}$. 
We show in~\Cref{subsec:connections} that, for settings where FQI was previously shown to succeed (e.g., under low distribution shift or Bellman completeness~\citep{wang2021what}), stability always holds and $\opnorm{P_\gamma}$ is never much larger than $1 / (1-\gamma)$, demonstrating how our bound recovers and unifies prior results.
However, we also find that, in general, this quantity provides a much \emph{sharper} notion of complexity for OPE. 
Indeed, there are simple instances where $\opnorm{P_\gamma}$ is $\mathcal{O}(1)$ for all $\gamma \in (0,1)$, but of course, $1/(1-\gamma)$ can be arbitrarily large.

The key insight behind this result is that, in the linear setting, FQI can be written as a power series in the empirical versions of the second moment matrices described in \Cref{eq:pop_defs}. More precisely,
	$\widehat{\theta}_T = \sum_{k=0}^T (\gamma \covhat^{-1} \crosshat)^k \covhat^{-1} \widehat{\theta}_{\phi,r}$
where $\widehat{\theta}_{\phi,r}$ is obtained by solving a regression for the rewards. The behavior of the algorithm is governed by the growth of these matrix powers. Using ideas from Lyapunov theory, we show that if stability holds, then these decay at a geometric rate governed by $\opnorm{P_\gamma}$ and FQI succeeds. On the other hand, if the spectral radius is greater than one, then these matrix powers grow exponentially, and FQI will drastically amplify any estimation errors. This leads to the \emph{necessity} of stability for FQI:

\newtheorem*{prop:fqi_lb_informal}{\Cref{prop:fqi_lb} (Informal)}
\begin{prop:fqi_lb_informal}  
If $\rho(\gamma \cov^{-1} \cross) > 1$, the variance of the FQI solution grows exponentially with the number of regression rounds $T$.
\end{prop:fqi_lb_informal}  

Turning to LSTD, while the solution is defined in terms of similar moment quantities to those relevant for FQI, it solves for $\theta_\gamma^\star$ in a more direct manner and hence its behavior is somewhat
different. 
We prove that LSTD succeeds if the following condition holds:

\begin{assumption}[Invertibility]
\label{ass:full_rank}
The matrices $\cov$ and $I-\gamma\cov^{-1} \cross$ are both full rank.
\end{assumption}

Our main result for LSTD is that under invertibility, $\theta^\star_\gamma$ is identifiable via LSTD as per the following informal theorem statement: 
\newtheorem*{thm:lspe_ub_informal}{\Cref{theorem:lstd_upper_bound} (Informal)}
\begin{thm:lspe_ub_informal}  
Let $\Qhat(s,a) = \phi(s,a)^\top \thetahatls$, where $\thetahatls$ is the LSTD solution. Under \Cref{ass:realizability,ass:full_rank}, as well as standard regularity assumptions for linear regression, if $n$ is large enough,
\begin{align*}
	\E_{\calD} |Q^\pi(s,a) - \widehat{Q}^\pi(s,a)| \lesssim \frac{1}{\smin(I -\whitecross)^2}  \sqrt{\frac{d\log(1/\delta)}{n}} 
\end{align*}
with probability $1-\delta$. Here, $\smin(\cdot)$ denotes the minimum singular value of a matrix.
\end{thm:lspe_ub_informal}

This result follows somewhat directly from a perturbation analysis of approximate solutions to the fixed point equation \Cref{eq:bellman_fixed_point}.
Perhaps surprisingly, we will see that invertibility is strictly
weaker than stability (\Cref{ass:stability}), which highlights a
fundamental distinction between these two methods. This comparison
also reveals that stability cannot be a necessary condition in any
algorithm-independent sense, since LSTD can succeed without
stability. However, complementing~\Cref{theorem:lstd_upper_bound}, we
prove that invertibility is necessary for a large class of natural
estimators, specifically those that rely on low-order moments of the
features and the regression function between features and the
rewards (this includes FQI and LSTD). The following lower bound shows
that the value function is \emph{unidentifiable} by these linear estimators if
invertibility does not hold.

\newtheorem*{thm:unidentifiable}{\Cref{theorem:lower_bound_lspe} (Informal)}
\begin{thm:unidentifiable}  Even in the limit of infinite data,
any OPE problem for which invertibility does not hold cannot be solved by a broad class of linear estimators, including FQI and LSTD.
\end{thm:unidentifiable}

Together with our previous results, this result completes our analysis of linear estimators for offline policy evaluation under linear realizability. We remark that our results are sharp in the sense that they stipulate exactly \emph{which} problems are solvable by linear estimators. They are not necessarily sharp in the sense that the associated statistical rates for each problem are optimal. We believe that establishing appropriate lower bounds for these problems is an important direction for future work.

\subsection{Related work}
\label{subsec:related_work}

\paragraph{RL with function approximation.}
Analyses of function approximation in reinforcement learning can be traced to the seminal papers of \cite{bellman1959functional, bellman1961approximation}, as well as \cite{reetz1977approximate} and \cite{ whitt1978approximations}. \cite{schweitzer1985generalized} were one of the  first to consider approximating value functions using linear combinations of some known set of features. 
More recently, a number of modeling assumptions---typically involving strong representational conditions on both the MDP and the features---that enable statistically efficient online RL with linear function approximation have been proposed, along with corresponding algorithms \citep{zanette2020learning, yang2020reinforcement, jin2020provably}.

\paragraph{FQI.} Introduced by \cite{ernst2005tree} and extended by \cite{riedmiller2005neural}, fitted Q-iteration has been analyzed several times in the context of  offline policy evaluation. Building off previous studies of approximate methods in dynamic programming \citep{antos2008learning,munos2007performance,gordon1999approximate}, \cite{chen2019information} establish sample complexity upper bounds for FQI assuming that the corresponding distributions and MDP satisfy concentrability \citep{munos2003error} and Bellman completeness \citep{szepesvari2005finite}. While concentrability conditions are  orthogonal to realizability assumptions, completeness is significantly stronger than mere realizability of value functions.
More recent work by \cite{wang2021what, wang2021Instabilities} adapts these results to the linear setting and additionally shows that a ``low distribution shift'' condition suffices for linear FQI.

\paragraph{LSTD.} Initial analysis of least squares temporal difference learning (LSTD) date back to the work of \citet{baird1995residual, bradtke1996linear,boyan1999least} and \citet{, nedic2003least}. Since then, the finite sample performance of the algorithm has been analyzed by \citet{lazaric2012finite,bhandari2018finite,duan2021optimal} and its behavior in the offline setting studied by \citet{yu2010convergence,li2021accelerated, mou2020OptimalOI,mou2021optimal,pires2012StatisticalLE}. \cite{tu2018least} analyze on-policy LSTD for the LQR setting. \cite{miyaguchi2021asymptotically} studies the behavior of LSTD for OPE in settings where the value function is only approximately linearly realizable in a known feature mapping $\phi$. We evaluate our contributions in light of these previous works in \Cref{subsec:invertibility_connections}.

\paragraph{Other OPE estimators.}

Apart from these methods, researchers have studied ``min-max'' algorithms for OPE which estimate the value of the underlying policy using ideas from the importance sampling literature \citep{liu2018breaking,uehara2020minimax,yin2020asymptotically}. 
\cite{xie2021batch}  establish formal guarantees for the BVFT algorithm which carries out policy evaluation for general \emph{nonlinear} function classes assuming realizability, albeit under stronger notions of data coverage (see \Cref{ass:concentrability}). 
Recent work by \cite{zhan2022offline} extends this line of research. They introduce a new algorithm which works under weaker data coverage assumptions than those in \cite{xie2021batch}. However, to do so they require additional assumptions on the expressivity of the underlying class of function approximators. 
In particular, \cite{zhan2022offline}, and the class of minimax algorithms more broadly, rely on a function class that can (at a minimum) realize the state-occupancy density ratio between the distribution induced by the policy $\pi$ and the offline distribution $\calD$, which is a distinct condition from linear realizability of $\Qpi$.

\paragraph{Lower bounds under linear realizability.} 
For the finite horizon, policy evaluation setting, \cite{wang2021what} illustrate how exponential dependence on the horizon is unavoidable, even if the offline covariance matrix is robustly full rank. 
Since then, these bounds have been extended to the discounted, infinite horizon case by \cite{amortila2020variant} and \cite{zanette2021exponential}. Importantly, \cite{amortila2020variant} establish that OPE can be information-theoretically intractable, even if: 1) all features are bounded, 2) $\cov$ is full rank, and 3) the learner has access to infinitely many samples drawn as in \Cref{subsec:prob_def}. 
Analogous negative results for online or generative-model settings have been shown to hold even in the presence of a constant suboptimality gap \citep{wang2021exponential} or polynomially large action sets \citep{weisz2021exponential,weisz2021tensorplan}.
\cite{duan2020minimax} prove lower bounds for OPE which hold for general function classes. 
\cite{foster2021offline} illustrate that polynomially many samples in the size of the state space are necessary for offline policy evaluation, even if concentrability and realizability both hold. In summary, a clean characterization of when offline policy evaluation is tractable using linear function approximation has, so far, proven to be quite elusive.

\section{Preliminaries}
\label{sec:preliminaries}
Before delving into our main results, we review some of the relevant definitions and preliminaries. 

\paragraph{Notation.} We use $s \in \states$ and $a \in \actions$ to denote states and actions, $\top$ to denote vector or matrix transposes, and $\dagger$ to denote pseudoinverses. 
For a matrix $X$, we let $\cond(X) := \sigma_{\max}(X) / \smin(X)$ denote its condition number, the ratio between the largest and smallest singular values $\sigma(\cdot)$. For symmetric matrices, $A$ and $B$, we use $A\succeq B$ if $A-B$ is positive semidefinite. 
We let $\rho(X) := \max_i |\lambda_i(X)|$ be the spectral radius of a matrix $X$ where $\lambda_i$ are the eigenvalues.\footnote{Recall that for square, but non-symmetric matrices $A$, it is in general \emph{not} true that $\rho(A) = \sigma_{\max}(A)$. However, $\rho(A) \leq \sigma_{\max}(A)$ does always hold.} 
We say that a matrix is \emph{stable} if its spectral radius is strictly smaller than 1. For square, stable matrices $A$, we let $\dlyaps{A}$ be the solution, over $X$, to the discrete-time Lyapunov equation: $X = A^\top X A +I$.
This equation has a solution if and only if  $\rho(A) < 1$ \citep{callier2012linear}. If the solution exists, it admits the closed-form expression $\dlyaps{A} = \sum_{j=0}^\infty (A^\top)^j A^j$. Lastly, we say $a \lesssim b$ if $a \leq c \cdot b$ for some universal constant $c$.

We define the next state-action covariance $\nextcov$ and the distribution shift coefficient $\cds$ as
\begin{align}
\label{eq:nextcov_cds_def}
	\nextcov \defeq \underset{\substack{(s,a) \sim \calD\\ s' \sim P(\cdot \mid s,a), \; a' \sim \pi(s')}}{\E} \left[\phi(s', a') \phi(s', a')^\top \right], \quad \cds \defeq \inf \{\beta > 0: \nextcov \preceq \beta \cov \}.
\end{align} 
Note that $\cds$ is guaranteed to be finite if $\cov$ is full rank. Given a dataset $\{(s_i, a_i, r(s_i, a_i), s_i', a_i') \}_{i=1}^n$ of $n$ i.i.d. data points drawn according to the data generating process described in \Cref{subsec:prob_def}, we define the empirical counterparts of the second-moment matrices defined in \Cref{eq:pop_defs},
\begin{align}
\label{eq:empirical_matrices}
\covhat \defeq \frac{1}{n}\sum_{i=1}^n \phi(s_i, a_i) \phi(s_i, a_i)^\top,  \quad \crosshat \defeq \frac{1}{n} \sum_{i=1}^n \phi(s_i, a_i)\phi(s_i', a_i')^\top,
\end{align}
as well as the true, and empirical, mean feature-reward vectors:
\begin{align}
\label{eq:thetaphir}
	 \theta_{\phi, r} \defeq \E_{\calD} \phi(s, a) r(s, a), \quad \widehat{\theta}_{\phi, r} \defeq \frac{1}{n}\sum_{i=1}^n \phi(s_i,a_i) r(s_i, a_i).
\end{align}
\paragraph{Linear regression.} Next, we introduce moment-type quantities that arise in our analysis of linear regression. 
Here, we adopt the approach from~\cite{hsu2012random}, however, other approaches for analyzing linear regression will yield the same qualitative results. 
In particular, we make use of the statistical leverages $\levcov$ and $\levnext$. These quantities correspond to the maximum length of features, $\phi(s,a)$ and $\phi(s',a')$, when measured in the (inverse) covariance norm. Intuitively, they capture the worst-case coverage of the offline distribution $\cD$ over directions in feature space.

\begin{align}
\label{eq:cds_def}
\levcov \defeq \underset{(s,a) \in \textrm{supp}(\calD)}{\sup} \twonorm{\cov^{-1/2} \phi(s,a)}, \quad \levnext := \underset{\substack{(s,a) \in \supp(\calD),\\ s'\in \supp(P(\cdot\mid (s,a)),\; a' \in \supp(\pi(s'))}}{\sup}  \twonorm{\cov^{-1/2} \phi(s', a')}.
\end{align}
In addition, we define the variances $\varcov, \rewardvar$, and $\crossvar$ where,
\begin{align}
\label{eq:varcov_def}
\varcov \defeq \opnorm{\E(\cov^{-1/2} \phi(s,a) \phi(s,a)^\top \cov^{-1/2})^2 - I},\quad
\rewardvar \defeq  \E\twonorm{\cov^{-1/2} \phi(s,a) r(s,a)}^2 - \twonorm{\cov^{-1/2} \tpr}^2,
\end{align}
and $\crossvar$ is the maximum of the following two quantities,
\begin{align}
\label{eq:cross_var_def}
\sup_{\twonorm{v}=1} \E\left(v^\top \cov^{-1/2}\phi(s', a')\right)^2 \twonorm{\cov^{-1/2}\phi(s,a)}^2  - \twonorm{\cov^{-1/2}\cross^\top \cov^{1/2} v}^2 \\
 \sup_{\twonorm{v}=1} \E\left(v^\top \cov^{-1/2}\phi(s,a)\right)^2 \twonorm{\cov^{-1/2}\phi(s', a')}^2  - \twonorm{\cov^{-1/2}\cross \cov^{1/2} v}^2.
\end{align}
In \Cref{subsec:variance_bounds}, we prove that $\crossvar$ and $\varcov$ can always be upper bounded in terms of the statistical leverages and the  coefficient $\cds$.\footnote{On the other hand, $\rewardvar$ is always upper bounded by $d$.} 
However, they can be much smaller in some settings.\footnote{For example, tighter bounds can be achieved if the distributions are hypercontractive, see~\Cref{subsec:variance_bounds}.} Therefore, for the sake of generality, we opt to state our bounds in terms of these quantities. Informally, these variance terms measure how much the corresponding matrices or vectors vary from their means, in the $\cov^{-1/2}$ geometry.

 Throughout our analysis of methods for offline policy evaluation, we will repeatedly make use of the following concentration result:
\begin{lemma}
\label{lemma:main_concentration_result}
For all $n \gtrsim \levcov^2 \log(d/\delta)$, define the estimation errors,
\begin{align}
\label{eq:epsdef}
	\epsop \defeq \opnorm{\cov^{1/2}  (\gamma \covhat^{-1} \crosshat)\cov^{-1/2} - \gamma \cov^{-1/2}\cross \cov^{-1/2}}, \quad  \epsz \defeq \twonorm{\cov^{1/2} (\covhat^{-1}\widehat{\theta}_{\phi,r}- \cov^{-1} \theta_{\phi, r})}. 
\end{align}
With probability $1-\delta$, $\covhat$ is full rank and $\epsz$, $\epsop$ satisfy the following inequalities:
\begin{align*}
\epsop &\lesssim \sqrt{ \frac{ \max(\crossvar, \varcov \Cds) \log(d/\delta)}{n}} + \frac{\max(\Cds^{1/2} \levcov^2, \levcov \levnext) \log(d/\delta)}{n}\\
	\epsz  &\lesssim \sqrt{ \frac{  \max(\twonorm{ \cov^{-1/2}\tpr}^2\varcov, \;\rewardvar) \log(d/\delta)}{n}} + \frac{ \twonorm{\cov^{-1/2}\theta_{\phi,r}} \levcov^2 \log(d/\delta)}{n}. 
\end{align*}

\end{lemma}
Later on, we state our upper bounds on the policy evaluation error of FQI and LSTD in terms of these regression errors $\epsop$, $\epsz$, with the understanding that they satisfy the high probability upper bounds above.

\section{Fitted Q-Iteration}
\label{sec:main_results}
In this section, we present our first set of results illustrating how stability (\Cref{ass:stability}) characterizes the success of fitted Q-iteration for OPE under linear realizability of $\Qpi$. Following some initial remarks regarding the functional form of the FQI solution, 
in \Cref{subsec:sufficiency}, we present our upper bound on the estimation error of FQI. Later on, in \Cref{subsec:connections}, we illustrate how our Lyapunov stability analysis unifies previous studies of when FQI succeeds and conclude by discussing lower bounds and limitations of the algorithm in \Cref{subsec:lower_bounds}.

\paragraph{FQI preliminaries.}
From examining the definition of FQI in \Cref{eq:fqi_def}, we see that, at the population level, the algorithm develops the recursion:
\begin{align*}
\theta_{t+1} = \gamma \cov^{-1} \cross \theta_t + \cov^{-1} \theta_{\phi,r}.
\end{align*}
Unrolling the recursion above, and setting $\theta_0 = 0$, the $T$-step regression vector is equal to:\footnote{We initialize at 0 for simplicity, but this is not fundamental for the overall analysis of FQI.} 
\begin{align}
\label{eq:tstep_fqi}
\theta_T = \sum_{k=0}^T (\gamma \cov^{-1} \cross)^{k} \cov^{-1}\theta_{\phi,r}.
\end{align} 
Linear realizability of $Q^\pi$ (\Cref{ass:realizability}) implies that the true weight vector $\theta_\gamma^\star$ satisfies the equation,
\begin{align}
\label{eq:thetastar}
	\cov \theta_\gamma^\star  =\theta_{\phi, r}  + \gamma \cross \theta_\gamma^\star.
\end{align}
Hence, if $I-\gamma \cov^{-1}\cross$ is invertible, then $\theta_\gamma^\star = (I - \gamma \cov^{-1} \cross)^{-1} \cov^{-1} \theta_{\phi,r}$. We now recall the following fact:
\begin{fact}
\label{fact:neumann_series}
If $\rho(A) < 1$, then the matrix $(I-A)$ is invertible. Moreover, $(I-A)^{-1} = \sum_{k=0}^\infty A^k$.
\end{fact}
Using this, along with the observation that the spectrum of a matrix is invariant to the choice of basis, we see that if stability (\Cref{ass:stability}) holds, then the vector $\theta^\star_\gamma$ can also be written as a power series:
\begin{align}
\label{eq:tgs_series}
\theta_\gamma^\star = \sum_{k=0}^\infty(\gamma \cov^{-1} \cross)^k \cov^{-1} \theta_{\phi,r}. 	
\end{align}

One of the key insights tying stability and FQI is that, regardless of whether $\gamma \cov^{-1} \cross$ is stable, the FQI solution at the population level is \emph{always} equal to the power series in \Cref{eq:tstep_fqi}. 
If stability holds, performing infinitely many exact FQI updates converges to $\theta^\star_\gamma$.
However, $\theta^\star_\gamma$ is (in general) \emph{only} equal to this power series if stability holds, which hints at the necessity of this condition. With these connections between stability and the functional forms of FQI and $\theta^\star_\gamma$ in mind, we now present our upper bounds on the performance of this algorithm.

\subsection{Stability is sufficient for fitted Q-iteration}
\label{subsec:sufficiency}

\begin{theorem}
\label{theorem:main_result}
Assume that $\Qpi$ is linearly realizable (\Cref{ass:realizability}) and that stability holds (\Cref{ass:stability}). For $\epsop, \epsz$ defined as in \Cref{eq:epsdef}, if $n \gtrsim \levcov^2 \log(d / \delta)$ and $\epsop \leq 1 / (6 \opnorm{P_\gamma}^2)$, $T$-step FQI satisfies,
\begin{align}
\label{eq:fqi_guarantee}
				\twonorm{\cov^{1/2}(\widehat{\theta}_T - \theta^\star_\gamma)} &\lesssim \cond(P_\gamma)^{1/2}\opnorm{P_\gamma} \cdot \epsz + \cond(P_\gamma)\opnorm{P_\gamma}^2 \cdot \twonorm{\cov^{-1/2}\theta_{\phi,r}} \cdot  \epsop \notag \\
				&+  \cond(P_\gamma) \opnorm{P_\gamma} \cdot \twonorm{\cov^{1/2}\theta_{\phi,r}}\cdot \exp\left( -\frac{T+1}{2\opnorm{P_\gamma}}\right).
\end{align}
\end{theorem}
Let $\Qhat(s,a) \defeq \phi(s,a)^\top \widehat{\theta}_T$. 
Much like in standard analyses of linear regression, from \Cref{theorem:main_result} we immediately obtain: (1) a bound on $\E_{\calD} |Q^\pi(s,a) - \widehat{Q}(s,a)|$ via Jensen's inequality since $\E_{\calD} (Q^\pi(s,a) - \Qhat(s,a))^2  = \twonorm{\cov^{1/2}(\widehat{\theta}_T - \theta^\star_\gamma)}^2$ and 
(2) a bound on $|\Qpi(s,a) - \Qhat(s,a)|$ for any $(s,a)$ pair since $|\Qpi(s,a) - \Qhat(s,a)| \leq \twonorm{\cov^{-1/2}\phi(s,a)} \twonorm{\cov^{1/2}(\widehat{\theta}_T - \theta_\gamma^\star)}$ via Cauchy-Schwarz.


We defer the full proof to \Cref{subsec:proof_main_theorem} and instead summarize the key steps here. 
The theorem is essentially a perturbation bound which distinguishes between two sources of error in policy evaluation for FQI: $\epsz$ which captures errors in learning the rewards, and the dominant error, $\epsop$, which comes from estimating the transitions. 
Since under stability, we can write the true vector $\theta_\gamma^\star$ as a power series in second moment matrices (see \Cref{eq:tgs_series}), 
and since $\widehat{\theta}_T$ is by definition a truncated power series in the empirical counterparts of these matrices,
we can show that the error between $\theta_\gamma^\star$ and $\widehat{\theta}_T$ is bounded by the operator norm of two power series: one in  $(\whitecross)^k$ and the other in $(\gamma \covhat^{-1/2} \crosshat \covhat^{-1/2})^k$.
Lyapunov arguments directly show that the powers of $(\whitecross)$ decay exponentially in $k$ since the matrix is stable.
For the empirical version, we use the fact that any stable matrix $A$ has nontrivial \emph{stability margin}: for small enough perturbations $\Delta$, matrices of the form $A+\Delta$ satisfy similar decay rates to $A$.
Thus, we can bound the two power series by simple geometric series and the perturbation bound follows. 

We now highlight some of the salient aspects of the bound.

\paragraph{Coordinate invariance.} The bound in \Cref{theorem:main_result} is \emph{coordinate-free}, in the sense that all problem quantities are invariant to the basis in which one chooses to represent the features. Linear realizability states that $Q^\pi(s,a) = \phi(s,a)^\top \theta_\gamma^\star$. 
Consequently, for any invertible matrix $L$, it also holds that $Q^\pi(s,a) = \tphi(s,a)^\top  \widetilde{\theta}_\gamma^\star$ where, 
\begin{align*}
\tphi(\cdot) = L \phi(\cdot)\text{ and }\widetilde{\theta}_\gamma^\star = L^{-1} \theta_\gamma^\star.
\end{align*}
 
Observe that the regression errors ($\epsz$ and $\epsop$) in the data norm, the geometry induced by $\cov$, do not depend on the choice of  matrix $L$, since the variances and statistical leverages are invariant to the coordinate system (see \Cref{lemma:main_concentration_result}). 
The invariance of $\opnorm{P_\gamma}$ and $\cond(P_\gamma)$ is perhaps less straightforward, but verified in the following proposition:
\begin{proposition}
\label{prop:coordinate_free}
Let $L \in \R^{d \times d}$ be an invertible matrix and let $\tphi(\cdot) = L \phi(\cdot)$ be the feature mapping in the new coordinates. Now, define $\widetilde{P}_\gamma \defeq \dlyaps{\gamma \tcov^{-1/2} \tcross \tcov^{-1/2}}$, where
\begin{align}
\label{eq:new_coord_mats}
\tcov \defeq \E_{(s,a)\sim \calD} \tphi(s,a) \tphi(s,a)^\top\text{, and} \quad \tcross \defeq  \underset{\substack{(s,a) \sim \calD\\ s' \sim P(\cdot \mid s,a), \; a' \sim \pi(s')}}{\E} \tphi(s, a)) \tphi(s', a')^\top.
\end{align}
Then, $\opnorm{P_\gamma} = \opnorm{\widetilde{P}_\gamma}$ and $\cond(P_\gamma) = \cond(\widetilde{P}_\gamma)$. Furthermore, $$\gamma \tcov^{-1/2} \tcross \tcov^{-1/2} = \gamma U \cov^{-1/2} \cross \cov^{-1/2} U^\top,$$ where $U \in \R^{d \times d}$ is an orthogonal matrix.
\end{proposition}

\paragraph{Sharpness of $\opnorm{P_\gamma}$ vs $1 / (1-\gamma)$.} Apart from showing how stability is sufficient for offline policy evaluation under linear realizability, another highlight of \Cref{theorem:main_result} is that it introduces a new measure of problem complexity, $\opnorm{P_\gamma}$, which is in general significantly sharper than previous complexity measures traditionally considered in the literature, such as the effective horizon, $1 / (1- \gamma)$. 
The difference between these two quantities is evident even in very simple settings:

Consider the following MDP (with no actions), where arrows denote transition probabilities:
\begin{equation}
\label{tikz:sharp}
\hbox{
\begin{tikzpicture}[scale=0.15]
\tikzstyle{every node}+=[inner sep=0pt]
\draw [black] (30.2,-27.5) circle (3);
\draw (30.2,-27.5) node {$s_0$};
\draw [black] (46.4,-27.5) circle (3);
\draw (46.4,-27.5) node {$s_1$};
\draw [black] (49.08,-26.177) arc (144:-144:2.25);
\draw (53.65,-27.5) node [right] {$1$};
\fill [black] (49.08,-28.82) -- (49.43,-29.7) -- (50.02,-28.89);
\draw [black] (33.2,-27.5) -- (43.4,-27.5);
\fill [black] (43.4,-27.5) -- (42.6,-27) -- (42.6,-28);
\draw (38.3,-28) node [below] {$1-p$};
\draw [black] (27.52,-28.823) arc (-36:-324:2.25);
\draw (22.95,-27.5) node [left] {$p$};
\fill [black] (27.52,-26.18) -- (27.17,-25.3) -- (26.58,-26.11);
\end{tikzpicture}}
\end{equation}

If $\E r(s_0)\neq 0$ and $\E r(s_1) =0$, realizability holds with 1 dimensional features: $\phi(s_0)=1$ and $\phi(s_1) =0$. For $\calD$ supported just on $s_0$, then $\gamma \cov^{-1} \cross = p \gamma$, and $P_\gamma = 1 / (1- (p\gamma)^2)$. If $p \leq 0.7$, then for all $\gamma \in (0,1)$, $\opnorm{P_\gamma} \leq 2$, but $(1-\gamma)^{-1}$ can be arbitrarily large as $\gamma \rightarrow 1$.

This example illustrates how there are problems for which $\opnorm{P_\gamma}$ is significantly smaller than $1 / (1-\gamma)$. In the next subsection, we complement this result by illustrating how for settings where  FQI was previously shown to succeed, $\opnorm{P_\gamma}$ is in fact never much worse than $1/ (1-\gamma)$. Taken together, these results demonstrate how $\opnorm{P_\gamma}$ provides a sharper notion of the statistical complexity  of OPE than $1 / (1-\gamma)$.


\subsection{Contextualizing Lyapunov stability}
\label{subsec:connections}

Having presented our analysis of fitted Q-iteration through the lens of Lyapunov stability, we now illustrate how this perspective unifies previously disparate analyses of FQI for offline policy evaluation. The central message of this subsection is that the previously proposed conditions which guarantee that FQI will succeed at offline policy evaluation directly imply our key assumption that $\gamma \cov^{-1} \cross$ is stable. 

Before discussing these connections, we present the following lemma which is closely related to~\Cref{theorem:main_result}. It upper bounds the error of FQI assuming particular decay rates  on the powers of the whitened cross-covariance matrix.
Although the proof is essentially identical to the previous result, we can obtain 
  sharper results assuming particular rates of decay, which will be helpful for
  later comparisons.
  
\begin{lemma}
\label{prop:contractive_case}
Assume $n \gtrsim \levcov^2 \log(d / \delta)$ and let $\epsop$ and $\epsz$ be defined as in \Cref{eq:epsdef}. Under the same assumptions as \Cref{theorem:main_result}, if there exist $\alpha > 0$ and $\beta \in (0,1)$ such that for all $k \geq 0$,
\begin{align}
\label{eq:growthrates}
\opnorm{(\gamma \cov^{-1/2}\cross \cov^{-1/2})^k} \leq \alpha \cdot \beta^{k},
\end{align}
then the $T$-step FQI solution satisfies the following guarantee. With probability $1-\delta$, if $\epsop \leq \frac{(1 - \beta)}{2\alpha}$,
\begin{align}
\label{eq:alpha_beta}
\twonorm{\cov^{1/2}(\widehat{\theta}_T - \theta_\gamma^\star)}\lesssim   \epsz \cdot  \frac{\alpha}{1-\beta} + \epsop \cdot \twonorm{\cov^{1/2} \tpr}  \frac{\alpha^2}{(1-\beta)^2}+ \twonorm{\cov^{1/2} \tpr}  \frac{\alpha}{1-\beta} \cdot \beta^{T+1}.
\end{align}
\end{lemma}

Throughout this section, we will present corollaries of this result, which can be viewed as specializations of \Cref{theorem:main_result} to particular settings. 
In each case, we will focus on discussing variants of the perturbation bound (\Cref{eq:alpha_beta}) which hold under the specific assumptions.

\subsubsection{Low distribution shift implies stability}

Recent work by \cite{wang2021Instabilities} shows that FQI succeeds at OPE for infinite horizon, discounted problems if there is low distribution shift. More formally, they prove offline evaluation is tractable if the  offline covariance $\cov$ has good coverage over the next state covariance $\nextcov$ as per the following assumption.
\begin{assumption}[Low Distribution Shift]
\label{ass:low_ds}
There is low distribution shift if $\cds < 1/\gamma^2$.
\end{assumption}
Note that if $\calD$ is the stationary measure for $\pi$, then $\cov = \nextcov$ and \Cref{ass:low_ds} holds with $\cds=1$ (recall the definition of $\cds$ in \Cref{eq:nextcov_cds_def}). Under this low distribution shift condition, we prove: 

\begin{corollary}
\label{corollary:low_distribution_shift}
If there is low distribution shift (\Cref{ass:low_ds}) and if $\cov$ is full rank, then for all $j \geq 0$, 
\begin{align}
\label{eq:lds_power_bounds}
\opnorm{(\gamma \cov^{-1/2} \cross \cov^{-1/2})^j} \leq (\sqrt{\cds \gamma^2})^j.
\end{align}
Hence, $\opnorm{P_\gamma}\leq 1 / (1 - \gamma\sqrt{\cds})$ and  \Cref{ass:stability} holds.\footnote{A matrix $A$ is stable if and only if $\lim_{k\rightarrow \infty} A^k=0$.} Furthermore, for $\gammads := \gamma\sqrt{\cds}$, if $\Qpi$ is linearly realizable (\Cref{ass:realizability}), $n \gtrsim \levcov^2 \log(d / \delta)$, and $\epsop \leq 1/2 (1 - \gammads)$, then $T$-step FQI satisfies:  
\begin{align*}
				\twonorm{\cov^{1/2}(\widehat{\theta}_T - \theta^\star_\gamma)} \lesssim \frac{1}{1-\gammads} \epsz +  \frac{1}{(1-\gammads)^2} \twonorm{\cov^{-1/2}\theta_{\phi,r}} \cdot \epsop +  \twonorm{\cov^{-1/2}\theta_{\phi,r}} \frac{1}{1-\gammads} \gammads^{T+1}.
\end{align*}
\end{corollary}
While low distribution shift implies stability, the converse is not true. It is not hard to come up with examples where $\whitecross$ is stable, yet the distribution shift coefficient is larger than $1/\gamma^2$. We present such an example later on in \Cref{prop:counterexamples}.

\subsubsection{Bellman completeness implies stability}
\label{subsec:completeness}
In addition to the low-distribution shift setting, FQI is known to succeed in both finite horizon and discounted, infinite horizon settings under a representational condition known as Bellman completeness \citep{szepesvari2005finite,wang2021what, wang2021Instabilities}:

\begin{assumption}[Bellman completeness]
\label{ass:completeness}
A feature map $\phi$ is Bellman complete for an MDP $\calM$, if for all $\theta \in \R^d$, there exists a vector $\theta'$ such that for all $(s,a)\in \states \times \actions$,
\begin{align*}
	\phi(s,a)^\top \theta' = \E \left[r(s,a)\right] +\gamma  \underset{s' \sim P(\cdot\mid s,a), a' \sim \pi(s')}{\E} \phi(s', a')^\top \theta .
\end{align*}
\end{assumption}
Intuitively, completeness asserts that Bellman backups of linear functions of the features again lie in the span of the features.
It has previously been observed \citep{wang2021what, wang2021Instabilities} that completeness implies a certain ``non-expansiveness'' of Bellman backups. 
This non-expansiveness is the key step towards establishing the connection to stability and is formalized in the following result:
\begin{corollary}
\label{corollary:completeness_case}
If $\phi$ is Bellman complete (\Cref{ass:completeness}) and $\cov$ is full rank, then for all $j \geq 0$, 
\begin{align}
\label{eq:completeness_power_bounds}
\opnorm{(\gamma\cov^{-1/2} \cross \cov^{-1/2})^j} \leq \levcov \gamma^j.
\end{align}
Hence, $\opnorm{P_\gamma} \leq \levcov / (1 -\gamma)$, and \Cref{ass:stability} holds. Furthermore, if $\Qpi$ is linearly realizable (\Cref{ass:realizability}), $n \gtrsim \levcov^2 \log(d / \delta)$, and $\epsop \leq (1-\gamma) / (2\levcov)$,  $T$-step FQI satisfies:  
\footnote{Completeness implies realizability of rewards which in turn implies $\twonorm{\cov^{-1/2} \theta_{\phi,r}}^2 = \E r(s,a)^2\leq 1$, see \Cref{lemma:reward_realizability}.}  
\begin{align*}
	\twonorm{\cov^{1/2}(\widehat{\theta}_T - \theta_\gamma^\star)} \lesssim \frac{\levcov}{1-\gamma}\cdot \epsz + \frac{\levcov^2}{(1-\gamma)^2} \epsop + \frac{\levcov}{1 - \gamma} \gamma^{T+1}.
\end{align*}
\end{corollary}

Again, as with low distribution shift setting, the converse statement is not true. There are OPE instances which are stable, but not Bellman complete (\Cref{prop:counterexamples})

\paragraph{The tabular case.} To help contextualize this result, and build some intuition between Bellman completeness and stability, we can consider the case of the tabular MDP. 
The tabular MDP is perhaps the simplest setting in which the Bellman completeness holds. 
In our setup, it means that $\states$ and $\actions$ are both finite sets and that the feature mapping is equal to $\phi(s,a) = e_{sa} \in \R^{|\states| |\actions|}$ for all $s$ and $a$ (each input maps to a distinct standard basis vector). The matrix $\cov$ being full rank means that every pair $(s,a) \in \states 
\times \actions$ is in the support of the offline distribution $\cD$. A direct calculation shows that 
\begin{align*}
	\gamma \cov^{-1} \cross \cov^{-1} = \gamma P^{\pi},
\end{align*}
where $P^{\pi} \in \R^{|\states||\actions| \times |\states||\actions|}$ is a row-stochastic matrix with nonnegative entries. Each row in this matrix is indexed by a  pair $(s,a)$. Entries in each row describe the probability that the next state action pair is $(s', a')$ given that the current pair is $(s, a)$. Because the spectral radius of any stochastic matrix is 1, when we multiply by $\gamma$, we get that $\rho(\whitecross) < 1$ and stability holds. 



\subsection{Stability is necessary for fitted Q-iteration}
\label{subsec:lower_bounds}
We conclude our analysis of FQI  by showing that our characterization of when the algorithm succeeds is   exactly sharp, in an instance-dependent sense. 
If stability fails that is, $\rho(\gamma \cov^{-1} \cross) > 1$, then estimation procedures  of this sort are guaranteed to have exponentially large variance.
\begin{proposition}
\label{prop:fqi_lb}
Let $\calM$ be any infinite horizon, discounted MDP with corresponding offline distribution $\calD$ which satisfies the following properties: $\cov$ is full rank and $\gamma\cov^{-1} \cross$ has an eigenvalue $\lambda$ with $|\lambda| > 1$.
Then, approximations of the $T$-step FQI solution, $\Qhat(s,a) = \phi(s,a)^\top \widehat{\theta}_T$ where,
\begin{align*}
\widehat{\theta}_T \defeq \sum_{k=0}^T (\gamma \cov^{-1} \cross)^k \cov^{-1}\widehat{\theta}_{\phi, r}, \quad \widehat{\theta}_{\phi,r} \defeq \theta_{\phi, r}^\star + z,
\end{align*}
and $z$ is a zero-mean, random vector satisfying $\Lambda \defeq \E zz^\top \succ 0$, have exponentially large variance, 
\begin{align*}
	\E \twonorm{\widehat{\theta}_T - \E \widehat{\theta}_T}^2 \geq \smin(\Lambda) \cdot \left( \frac{\lambda^{T+1} - 1}{\lambda -1}\right)^2. 
\end{align*}
\end{proposition}

This proposition corroborates empirical findings on the instability of FQI by \cite{wang2021Instabilities} and shows that an idealized variant of FQI incurs exponentially large variance (in the number of rounds $T$) for an instance that results in an unstable ``backup operator'' $\gamma \cov^{-1} \cross$. 
  By standard bias-variance decomposition, this directly implies exponentially large error for estimating $\theta_\gamma^\star$. Although, note that since stability does not hold, there is no guarantee that $\theta_\gamma^\star$ can be written as a power series, so it may not even be the limiting solution of population FQI as discussed at the beginning of this section.
  
  The algorithm is idealized in two senses, both of which are relatively minor. First, it has perfect knowledge of $\cov$ and $\cross$ which does not happen in practice, but is favorable to the algorithm, resulting in a stronger lower bound.
  Second, the error in estimating the reward is assumed to have a full-rank covariance; this arises naturally whenever rewards are perturbed with centered Gaussian noise since $\cov$ is full rank.
  Thus, the result shows that even when the dynamics are known, errors in estimating the rewards will be exponentially magnified, resulting in overall divergence of the algorithm. 

While the theorem does not consider the marginally stable case where $\rho(\gamma \cov^{-1}\cross) = 1$, we note in the proof that if the spectral radius is exactly one, the variance can grow at least linearly with $T$. 
However, marginal stability introduces other issues as we illustrate later on.

At this point, it is natural to wonder whether stability is necessary not just for the success of this algorithm, but rather for the success of \emph{any} algorithm at offline policy evaluation. It turns out that this is is not the case. As we will show in the following section, least squares temporal difference learning works under strictly weaker conditions than fitted Q-iteration.

\section{Least Squares Temporal Difference Learning}
\label{sec:lspe}
Building on our analysis of FQI, we now analyze how a closely related algorithm, least squares temporal difference learning, overcomes some of its shortcomings in the context of offline policy evaluation. Similarly to the previous section, we start by illustrating how invertibility is sufficient for LSTD in \Cref{subsec:lspe_upper}, and discuss connections to previous sufficient conditions in \Cref{subsec:invertibility_connections}. 
Lastly, we conclude in \Cref{subsec:moment_limits}  by presenting lower bounds which show that if invertibility does not hold, then the offline policy evaluation problem cannot be solved using linear estimators (FQI and LSTD being special cases), even asymptotically.

\subsection{Invertibility is sufficient for LSTD}
\label{subsec:lspe_upper}
\begin{theorem}
\label{theorem:lstd_upper_bound}
Assume that realizability and invertibility (\Cref{ass:realizability,ass:full_rank}) both hold and let $\epsz$, $\epsop$ be defined as in \Cref{eq:epsdef}. If $n \gtrsim \levcov^2 \log(d / \delta)$ and $\epsop \leq \smin(I - \gamma \cov^{-1/2}\cross\cov^{-1/2}) / 2$, then the LSTD solution, 
\begin{align*}
	\thetahatls \defeq (I - \gamma \covhat^{-1} \crosshat)^{\dagger} \covhat^{-1} \widehat{\theta}_{\phi,r},
\end{align*}
satisfies the following error guarantee:
\begin{align}
\label{eq:lstd_guarantee}
	\twonorm{\cov^{1/2}(\theta_\gamma^\star - \thetahatls)} \lesssim  \frac{1}{\smin(I - \gamma \cov^{-1/2}\cross\cov^{-1/2})} \cdot \epsz + \frac{1}{\smin(I - \gamma \cov^{-1/2}\cross\cov^{-1/2})^2} \twonorm{\cov^{-1/2} \theta_{\phi,r}}\cdot \epsop .
\end{align} 
\end{theorem}

As per our discussion immediately following \Cref{theorem:main_result}, the upper bound on $\twonorm{\cov^{1/2}(\theta_\gamma^\star - \widehat{\theta}_\gamma)}$ again directly implies guarantees on $|\Qpi(s,a) - \Qhat(s,a)|$, both pointwise and in expectation, where now $\Qhat(s,a) = \phi(s,a)^\top \thetahatls$. 
On a technical level, the proof follows from standard perturbation bounds on matrix inverses.

Our upper bound for LSTD has qualitatively similar properties to that presented for FQI in \Cref{theorem:main_result}.

\paragraph{A sharper notion of problem complexity.}

Much like $\opnorm{P_\gamma}$ for FQI, the magnitude of our upper bound for the policy evaluation error of LSTD is determined by an instance-dependent quantity: $1/\smin(I - \gamma \cov^{-1/2}\cross\cov^{-1/2})$. 
As before this term is: $(1)$ never much larger than $1/(1-\gamma)$ for settings where OPE was previously shown to be tractable (see the next subsection for further discussion of this point), and $(2)$ is often significantly smaller. 
For example, for the OPE instance detailed in \eqref{tikz:sharp}, if $p \leq .7$, then $1/\smin(I - \gamma \cov^{-1/2}\cross\cov^{-1/2}) \leq 4$ for all $\gamma \in (0,1)$. 

\paragraph{Coordinate invariance.} From \Cref{prop:coordinate_free}, we know that for any choice of full rank matrix $L$ and features $\widetilde{\phi}(\cdot) = L \phi(\cdot)$, the whitened cross-covariance in these new features, $\gamma \tcov^{-1/2} \tcross \tcov^{-1/2}$ (see definition in \Cref{eq:new_coord_mats}) is equal to $\gamma U \cov^{-1/2} \cross \cov^{-1/2}U^\top$ for some orthogonal matrix $U$. Since conjugating by an orthogonal matrix preserves singular values, $1/\smin(I - \gamma \cov^{-1/2}\cross\cov^{-1/2})$ is invariant to the choice of coordinates. 


\subsection{Contextualizing Invertibility}
\label{subsec:invertibility_connections}

Paralleling our discussion of stability for FQI, we now discuss how our notion of invertibility relates to previous conditions analyzed in the literature. Furthermore, we will present how stability implies invertibility, establishing a precise ``nesting'' between the classes of OPE problems which satisfy each condition.


\subsubsection{Stability $\subsetneq$ Invertibility}

\begin{proposition}
\label{prop:stability_invertibility}
If $\cov$ is full rank and $\whitecross$ is stable (\Cref{ass:stability}), then $I - \whitecross$ is  invertible (\Cref{ass:full_rank}). Furthermore, 
\begin{align}
\label{eq:min_py}
	\frac{1}{\smin(I - \whitecross)} \lesssim \cond(P_\gamma)^{1/2} \opnorm{P_\gamma}.
\end{align}
\end{proposition}

The main message of this proposition is twofold. 
First, for the case of linear function approximation, any OPE problem that is solvable via FQI, must also be solvable via LSTD. 
Second, from \Cref{eq:min_py} we see that main complexity measure for  \Cref{theorem:lstd_upper_bound}, $1 / \smin(I -\whitecross)$ is never larger than the corresponding upper bound for FQI in \Cref{theorem:main_result}, $\cond(P_\gamma)^{1/2} \opnorm{P_\gamma}$.

Interestingly enough, while stability implies invertibility, the converse is not true. There exist problems for which $I - \whitecross$ is invertible, but $\whitecross$ is not stable. For example, consider the following 2 state MDP, with no actions:
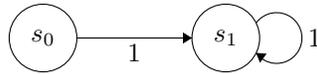
\begin{center}
\begin{tikzpicture}[scale=0.15]
\tikzstyle{every node}+=[inner sep=0pt]
\draw [black] (30.2,-27.5) circle (3);
\draw (30.2,-27.5) node {$s_0$};
\draw [black] (46.4,-27.5) circle (3);
\draw (46.4,-27.5) node {$s_1$};
\draw [black] (49.08,-26.177) arc (144:-144:2.25);
\draw (53.65,-27.5) node [right] {$1$};
\fill [black] (49.08,-28.82) -- (49.43,-29.7) -- (50.02,-28.89);
\draw [black] (33.2,-27.5) -- (43.4,-27.5);
\fill [black] (43.4,-27.5) -- (42.6,-27) -- (42.6,-28);
\draw (38.3,-28) node [below] {$1$};
\end{tikzpicture}
\end{center}
If we set $R(s_0) = R(s_1) = \textrm{Unif}(\{\pm 1\})$, and $\phi(s_0)=1$, $\phi(s_1)=2$, then this OPE instance is trivially linearly realizable with $\theta^\star_\gamma = 0$. 
If the offline distribution $\calD$ places mass $p$ on $s_0$ and $1-p$ on $s_1$, it is easy to see that  $I - \whitecross$ is invertible for all $p, \gamma \in (0,1)$. 
However, for $p=\gamma=.9$, $\whitecross$ is at least $3/2$, hence stability does not hold and FQI will necessarily diverge. Together, these results establish a separation between the set of problems solvable via FQI and those solvable via LSTD.%
\footnote{The careful reader might observe that the main reason why FQI fails in this example is that the algorithm is sensitive to the scale of the next state features. 
For instance, stability (and realizability) would hold if $|\phi(s_1)| < 1$.}

Moreover, for the set of previously analyzed settings where stability holds, we can establish quantitative upper bounds on $1/\smin(I - \gamma \cov^{-1/2}\cross\cov^{-1/2})$ illustrating how this quantity is comparable to $1/(1-\gamma)$.

\begin{corollary}
Assume $\cov \succ 0$. If there is low distribution shift (\Cref{ass:low_ds}), then for $\gammads := \gamma\sqrt{\cds}$,
\begin{align*}
\frac{1}{\smin(I - \gamma \cov^{-1/2}\cross\cov^{-1/2})} \leq \frac{1}{1 - \gammads}.
\end{align*}
Moreover, if Bellman completeness holds (\Cref{ass:completeness}), then
\begin{align*}
\frac{1}{\smin(I - \gamma \cov^{-1/2}\cross\cov^{-1/2})} \leq \frac{\levcov}{ 1 - \gamma}.
\end{align*}
\end{corollary}
This result follows from observing that $1 / \smin(I - \gamma \cov^{-1/2}\cross\cov^{-1/2}) = \opnorm{(I -\whitecross)^{-1}}$. Since stability holds for both of these settings, we can use \Cref{fact:neumann_series} to write $(I -\whitecross)^{-1}$ as an infinite power series in $\whitecross$. Applying the triangle inequality and the bounds from \Cref{eq:completeness_power_bounds,eq:lds_power_bounds} on the powers of $\whitecross$ finishes the proof of this corollary. 


\subsubsection{Other connections}

Recent work by \cite{mou2020OptimalOI} analyzes oracle inequalities for solving projected fixed point equations, of which the Bellman equation (\Cref{eq:thetastar}) is a special case. For the offline policy evaluation setting, they prove that a stochastic approximation variant of LSTD succeeds if the following condition holds:
\begin{assumption}[Symmetric Stability]
\label{ass:wenlong}
The matrix $\cov$ is full rank, and $\whitecross$ satisfies 
\begin{align*}
\kappa \defeq \frac{1}{2}\lambda_{\max}(\whitecross + (\whitecross)^\top)<1.
\end{align*}
\end{assumption}
Here, $\lambda_{\max}$ denotes the maximal eigenvalue of a matrix.\footnote{The matrix in \Cref{ass:wenlong} is symmetric so all eigenvalues are real and the maximum is well defined.} In their paper, the authors remark how \Cref{ass:wenlong} directly implies that $I - \whitecross$ is invertible. Amongst other quantities, their bounds scale with $1 / (1 - \kappa)$. This quantity is always at least as large as $1 / \smin(I - \whitecross)$.
\begin{proposition}
\label{prop:wenlong}
If \Cref{ass:wenlong} holds, then $I-\whitecross$ is invertible and
\begin{align*}
	\frac{1}{\smin(I - \whitecross)} \leq \frac{1}{1-\kappa}.
\end{align*}
\end{proposition}

Recent work by \cite{li2021accelerated} extends the stochastic approximation analysis from \cite{mou2020OptimalOI} to incorporate variance reduction techniques. Their upper bounds  directly assume invertibility, but also have explicit dependence $1/(1-\gamma)$ which can be quite loose in certain settings as detailed earlier. 

Apart from these analyses, \cite{kolter2011fixed} proves that LSTD succeeds in the offline setting if a certain linear matrix inequality holds:

\begin{assumption}[Contractivity]
\label{ass:kolter}
The matrix $\cov$ is full rank and together with  $\cross$ satisfies,
\begin{align*}
	\begin{bmatrix}
		\cov & \cross \\ 
		\cross^\top & \cov 
	\end{bmatrix} \succeq 0.
\end{align*}
\end{assumption}
A simple Schur complement argument illustrates that this assumption from \cite{kolter2011fixed} implies that the whitened cross covariance has \emph{operator norm} strictly less than 1. 
Since the spectral radius of a matrix is always smaller than its operator norm, this condition directly implies that $\whitecross$ is stable (\Cref{ass:stability}) and that $I- \whitecross$ is invertible (\Cref{ass:full_rank}). 

\begin{proposition}
\label{prop:kolter}
If \Cref{ass:kolter} holds, then $\opnorm{\whitecross}<1$ and stability holds.
\end{proposition}

As in the case of FQI, we see how our characterization of LSTD in terms of invertibility neatly unifies previous analyses of when this algorithm succeeds in the offline setting. Furthermore, our invertibility-based analysis strictly subsumes these previous studies. There exist problems for which  stability and invertibility hold but these other conditions (e.g.,~low distribution shift, Bellman completeness, etc.) do not.

\begin{proposition}
\label{prop:counterexamples}
For each of the following cases, there exists an offline policy evaluation problem defined by an MDP $\mathcal{M}$, an offline distribution $\cD$, and a  target policy $\pi$ such that $\Qpi$ is linearly realizable in a feature mapping $\phi$ (\Cref{ass:realizability} holds) where:
\begin{itemize}
	\item Stability and invertibility both hold, yet low distribution shift (\Cref{ass:low_ds}) does not. 
	\item Stability and invertibility both hold, yet Bellman completeness (\Cref{ass:completeness}) does not. 
	\item Stability and invertibility both hold, yet symmetric stability (\Cref{ass:wenlong}) does not. 
	\item Stability and invertibility both hold, yet contractivity (\Cref{ass:kolter}) does not. 
\end{itemize}
\end{proposition}

In short, there is a nontrivial gap between the problems we knew could be solved via previous analyses and the ones we know we can solve in light of our work.

\subsection{Invertibility is necessary for all linear estimators}
\label{subsec:moment_limits}

We finish our presentation of LSTD by proving that invertibility is not just sufficient, it is also strictly necessary for LSTD, as well as for a broad class of ``linear'' estimators. To do so, we first formally define what we mean by linear estimators:

\begin{definition}[Population Linear Estimator]
\label{def:linear_estimator}
Let $\mathsf{Alg}$ be a deterministic algorithm which given an infinite horizon, discounted MDP $\calM$, a distribution $\calD$ over $\states \times \actions$, and a policy $\pi$ returns a function $\Qhat: \states \times \actions \rightarrow \R$. Furthermore, let $(\calM, \calD, \pi)$ and $(\altm, \overline{\calD}, \overline{\pi})$ be two OPE instances such that: 
\begin{itemize}
	\item The corresponding action value functions $\Qpi, \bar{Q}^{\overline{\pi}}$ are both linearly realizable in a feature map $\phi$.
	\item The covariance, cross-covariance and mean feature-reward vectors (as defined in \Cref{eq:pop_defs,eq:thetaphir}) are identical in $(\calM, \calD, \pi)$ and $(\altm, \overline{\calD}, \overline{\pi})$: 
	\begin{align*}
		\altcov = \cov, \quad \altcross = \cross, \quad \altthetar = \theta_{\phi, r}, \quad \E_{\calD}\overline{r}(s,a) = \E_{\calD} r(s,a).
	\end{align*}
\end{itemize}
We say that $\mathsf{Alg}$ is a \emph{population linear estimator} if $\mathsf{Alg}(\calM, \calD, \pi) = \mathsf{Alg}(\altm, \overline{\calD}, \overline{\pi})$.
\end{definition}

While our focus has been on studying the finite sample performance of estimators for OPE, in this definition we choose to catalogue algorithms based on their asymptotic behavior so as to neatly abstract technical modifications like variance reduction. These techniques introduce differences in finite sample behaviors, but are not essential to the overall \emph{identifiability} concerns that are the focus of this subsection.

Intuitively, linear estimators are those whose population-level solution depends on the low-order moments of the data. These moments correspond to the quantities which appear in the solution to the projected Bellman equation:
\begin{align*}
\cov \theta_\gamma^\star  =\theta_{\phi, r}  + \gamma \cross \theta_\gamma^\star.
\end{align*} 
From their definitions in \Cref{eq:tstep_fqi,eq:lstd_def}, we see that common estimators such as LSTD and FQI both satisfy this definition. Interestingly, not all known, or least-squares-like, estimators are linear (e.g Bellman Residual Minimization). We will discuss these after presenting the lower bound.

\begin{theorem}
\label{theorem:lower_bound_lspe}
Let $\calM=(\states, \actions, P, R, \gamma)$ be any MDP with associated offline distribution $\calD$ with rewards uniformly bounded by 1 $(\text{i.e.,~}\sup_{(s,a) \in \states \times \actions} |r(s,a)| \leq 1)$ such that: 
\begin{itemize}
	\item $\Qpi(s,a)$ is linearly realizable in $\phi$.
	\item $\cov$ is full rank.
	\item $I-\gamma \cov^{-1}\cross$ is rank deficient.
\end{itemize}
Then, there exists a different MDP $\altm = (\states, \actions, P, \altreward, \gamma)$, with identical states, actions, and transitions, and whose reward distribution $\altreward$ is uniformly bounded by 2, such that for the same offline distribution $\calD$: 
\begin{itemize}
	\item The Q-function for $\pi$ in $\altm$, $\altqpi$, is linearly realizable in the same feature mapping $\phi$.
	\item The covariance, cross-covariance, next state covariance, and mean feature-reward vector in $\altm$ are identical to their counterparts in $\calM$: 
	\begin{align*}
		\altcov = \cov, \quad \altcross = \cross,\quad \altnext = \Sigma_{\mathrm{next}}, \quad \altthetar = \theta_{\phi, r}.	\end{align*}
	\item However, the $Q$ functions are different: $$\E_{\calD}(\Qpi(s,a) - \bar{Q}^\pi(s,a))^2 \gtrsim \smin(\cov)\; / \sup_{(s,a) \in \states \times \actions} \twonorm{\phi(s,a)}^2.$$
\end{itemize} 
Consequently, if we define $\mathsf{LE}$ as the set of population linear estimators which satisfy \Cref{def:linear_estimator}, we have that 
\begin{align*}
	\inf_{\mathsf{Alg} \in \mathsf{LE}} \sup_{(\calM', \calD', \pi') \in \calN} \E_{\calD} (Q'^{\pi}(s,a) - \Qhat(s,a))^2 \gtrsim \smin(\cov)\; / \sup_{(s,a) \in \states \times \actions} \twonorm{\phi(s,a)}^2.
\end{align*}
where $\Qhat = \mathsf{Alg}(\calM', \calD', \pi')$ and $\calN = \{ (\calM, \calD, \pi), (\altm, \calD, \pi)\}$
\end{theorem}

In other words, this theorem states that for \emph{any} OPE instance where $I-\gamma \cov^{-1}\cross$, or equivalently, $I - \whitecross$, is rank deficient, we can perturb the rewards to construct an alternative instance with matching low order moments. Consequently, any population linear estimator, such as LSTD or FQI, will  return the same estimate $\Qhat$ in both cases.
Yet, since the $Q$-functions are distinct, they will necessarily converge to the wrong answer in one case. Note that the alternative instance $\altm$ has identical states, actions, and transitions. Therefore, any function of these quantities, not just the ones explicitly listed above, will be the same in $\calM$ and $\altm$. 
Together with \Cref{theorem:lstd_upper_bound}, this result illustrates how our characterization of the settings where LSTD succeeds is exactly sharp in an instance-dependent (local) sense.

\subsubsection{Going beyond linear estimators}

\paragraph{Bellman residual minimization.} Bellman residual minimization attempts to estimate the value of a decision making policy by solving the following optimization problem, defined here at the population level:
\begin{align*}
	\theta_{\mathrm{BRM}} \in \argmin_{\theta}\underset{(s,a) \sim \calD,  s' \sim P(\cdot |s,a), a' \sim \pi(s')}{\E}(\phi(s,a)^\top \theta - r(s,a) - \gamma \cdot \phi(s',a')^\top \theta)^2.
\end{align*}
In the linear function approximation setting, the BRM solution is equal to 
\begin{align*}
	\tbrm = (\cov - \gamma \cross - \gamma \cross^\top + \gamma^2 \nextcov)^\dagger (\tpr - \gamma \E \phi(s', a') r(s, a)).
 \end{align*}
 The key difference with regards to previously analyzed estimators is that BRM depends on the correlation between the \emph{next} state feature vector $\phi(s', a')$ and the reward. However, FQI and LSTD only depend on the correlation $\tpr = \E \phi(s,a) r(s,a)$ between the \emph{current} state and the reward. 
 
 To the best of our knowledge, there is no exact characterization of when BRM succeeds at offline policy evaluation under linear realizability. In particular, it is not sufficient for the matrix,
 \begin{align*}
 \cov - \gamma \cross - \gamma \cross^\top + \gamma^2 \nextcov,
 \end{align*}
 to be invertible. On the other hand, it is well-known that BRM can be inconsistent if the dynamics of the MDP are not deterministic. In general, this algorithm requires use of the \emph{double sampling trick} and the ability to reset the environment to  particular states via a simulator. 
 We provide a more detailed discussion of these issues in \Cref{subsec:brm_counterexample} and refer the interested reader to \citep{baird1995residual, saleh2019deterministic}. 

\paragraph{Algorithm independent limits of OPE}

Given the negative result from \Cref{theorem:lower_bound_lspe}, a natural question to ask is:
what are the algorithm-independent limits for OPE under linear realizability?
We close this section with a brief discussion of how our work provides insight into this question.

We start by pointing out that there are settings where invertibility fails and for which offline policy evaluation is information-theoretically impossible. That is, OPE is not solvable regardless of the choice of estimator or the number of samples observed. This observation follows from the construction in \cite{amortila2020variant}. We reproduce their result for the sake of completeness:

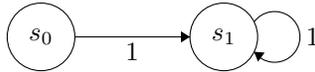
\begin{center}
\begin{tikzpicture}[scale=0.15]
\tikzstyle{every node}+=[inner sep=0pt]
\draw [black] (30.2,-27.5) circle (3);
\draw (30.2,-27.5) node {$s_0$};
\draw [black] (46.4,-27.5) circle (3);
\draw (46.4,-27.5) node {$s_1$};
\draw [black] (49.08,-26.177) arc (144:-144:2.25);
\draw (53.65,-27.5) node [right] {$1$};
\fill [black] (49.08,-28.82) -- (49.43,-29.7) -- (50.02,-28.89);
\draw [black] (33.2,-27.5) -- (43.4,-27.5);
\fill [black] (43.4,-27.5) -- (42.6,-27) -- (42.6,-28);
\draw (38.3,-28) node [below] {$1$};
\end{tikzpicture}
\end{center}

There are 2 states and no actions. The feature map is defined as $\phi(s_0)= \gamma$ and $\phi(s_1)=1$. The rewards are $\E r(s_0) = 0$ and $\E r(s_1) = r_\star \neq 0$. Realizability holds for any choice $r_\star$ with $\tsg = r_\star / (1-\gamma)$. If the offline distribution $\cD$ is supported just on $s_0$, then $\cov=\gamma^2$, $\cross = \gamma$ and $I - \gamma \cov^{-1} \cross =0$. Hence, invertibility fails for this problem. Furthermore, because the nonzero reward $r_\star$ is never observed under the offline distribution $\cD$, OPE is impossible even in the limit of infinite data. \footnote{We can check that invertibility holds if the distribution $\cD$ places nonzero mass on the second state $s_1$.}
In short, this example shows that if invertibility fails, then OPE cannot be solved in the worst case. However, there are problems where invertibility fails, yet offline policy evaluation is still possible via nonlinear estimators.

Introduced by \citet{xie2021batch}, the BVFT algorithm is a statistically, but not computationally, efficient algorithm for offline policy evaluation using a \emph{general} function class $\calF$ under two assumptions: (1) $\Qpi$ is realizable by a function in the class $\calF$ and (2) the offline distribution $\calD$ and the MDP dynamics satisfy a strong data coverage condition referred to as \emph{pushforward concentrability}. 

\begin{assumption}[Pushforward Concentrability, \cite{xie2021batch}]
\label{ass:concentrability}
An MDP $\calM$  and offline distribution $\calD$ satisfy pushforward concentrability if:
\begin{itemize}
	\item The offline distribution $\calD$ has strictly positive mass on all $(s,a) \in \states \times \actions$: $P_\calD(s,a)>0$.
	\item There exists a constant $1\leq C_A<\infty$ such that for any $(s,a) \in \states \times \actions$, $P_{\calD}(a \mid s) \geq 1 / C_A$.
	\item The exists a constant $0 < C_S < \infty$ such that for all $s,s' \in \states$ and $a \in \actions$.\footnote{We omit the last assumption on the initial state distribution from \cite{xie2021batch} as it is not essential for the purposes of our discussion.}
	\begin{align*}
		\frac{P(s'\mid s, a)}{P_{\calD}(s')} \leq C_S.
	\end{align*}
\end{itemize}
\end{assumption}

In the linear function approximation setting, realizability of $\Qpi$ in $\calF$ reduces to our realizability condition (\Cref{ass:realizability}). However, pushforward concentrability is in general distinct from stability or invertibility. 
That is, for problems that are linearly realizable, pushforward concentrability does not imply, nor is implied by, the assumption that $\smin(I - \whitecross) > 0$. 
Therefore, there exist settings where linear estimators may fail, yet BVFT can succeed and vice versa.

To see this, we consider a variation of the MDP defined just above. The dynamics are identical, but we alter the reward function and the feature mapping. In particular, here we choose the feature map $\phi(s_0)=1$ and $\phi(s_1) = 2 / \gamma$. If we set the rewards to have nonzero variance and satisfy $\E r(s_1) = r^\star$, $\E r(s_0) = \frac{-\gamma}{2(1-\gamma)}r^\star$, then this MDP is linearly realizable with $\theta^\star_\gamma = \frac{\gamma}{2(1-\gamma)}r^\star$. For any $\gamma \in (0,1)$, a simple continuity argument proves that there always exists a $p \in (0,1)$ such that if the offline distribution places mass $p$ on $s_0$ and $1-p$ on $s_1$, $\cov$ is full rank and $\whitecross=1$. Therefore, realizability and pushforward concentrability both hold, but invertibility does not. For the converse direction, it is not hard to see how one might construct examples where linear realizability and invertibility both hold, but \Cref{ass:concentrability} does not. The first condition asserting that $\calD$ be supported on all states and actions is particularly stringent.\footnote{In this construction, we have departed from our assumption that $\sup_{s,a} |r(s,a)|<1$ since $\E r(s,a)$ is on the order of $\Omega((1-\gamma)^{-1})$. However, the magnitude of the rewards should not affect the \emph{identifiability} of $\Qpi$, only the estimation rate for quantities like $\epsop$ and $\epsz$.}

Recall from the construction in \Cref{theorem:lower_bound_lspe}, that for any OPE instance where invertibility fails, the alternative $\altm$ has exactly the same states and transitions. Therefore, any estimator that outperforms linear methods must necessarily consider nonlinear or higher-order interactions between features and rewards. Interestingly enough, a simple tabular method, which ignores the feature mapping $\phi$ and directly estimates the rewards, successfully approximates the value function in this example.

\section{Offline Policy Evaluation without Realizability}
\label{sec:misspecfication}

Throughout our presentation thus far, our main focus has been on understanding exactly when and why various popular estimators succeed at offline policy evaluation, under the assumption that the action value function  \emph{exactly} satisfies the linear realizability condition. 
Of course, in practice, we might not expect linear realizability to hold exactly, but rather only approximately. 

As a sanity check, we therefore investigate how the performance of FQI and LSTD degrade if the relevant function approximation guarantees are weakened. Relative to previous results in this paper, the results in this section are more exploratory and speculative. We leave the problem of generating a more complete understanding of OPE under misspecification to future work.
For simplicity, here we analyze the behavior of these estimators under an $\ell_\infty$ guarantee on the error of the feature mapping $\phi$.

\begin{definition}[Approximate Realizability]
We define  $\tsi$ as the vector that minimizes the worst-case error with respect to $\Qpi$. Formally, $\tsi$ is the solution to the following optimization problem, where $\phi:\states \times \actions \rightarrow \R^d$: 
\begin{align}
\label{eq:tsi_def}
	\tsi \in \argmin_{\theta \in \R^d} \sup_{(s,a) \in \states \times \actions} |\Qpi(s,a) - \phi(s,a)^\top \theta|
\end{align}
We define the approximation error of $\tsi$ as, 
$\epsinf \defeq \min_{\theta \in \R^d} \sup_{(s,a) \in \states \times \actions} |\Qpi(s,a) - \phi(s,a)^\top \theta|$.
\end{definition}

Since the rewards are always bounded, $\epsinf$ is trivially always bounded by $1/(1-\gamma)$. 
On the other hand, if $\epsinf =0$, \cref{ass:realizability} holds, and we recover the linear realizability setting that has been the main focus of this paper. Values of $\epsinf$ interpolating between these two extremes measure the extent to which the value function $\Qpi$ can be expressed as a linear function of the features $\phi$, in a worst case sense.

Using this definition, we prove the following proposition which, together with \Cref{theorem:main_result,theorem:lstd_upper_bound}, bounds the error of FQI and LSTD under misspecification.


\begin{proposition}
\label{prop:misspecification}
Assume that invertibility (\Cref{ass:full_rank}) holds and let $\Qhat(s,a) = \phi(s,a)^\top \thetahat$ be an estimator satisfying, 
\begin{align*}
	\twonorm{\cov^{1/2}(\tfp - \thetahat)} \leq \epsfp \text{ for } \tfp \defeq (I - \whitecross)^{-1} \cov^{1/2} \tpr.
\end{align*}
Then, for any $(s,a) \in \states \times \actions$, 
\begin{align}
\label{eq:qpi_miss}
|\Qpi(s,a) - \Qhat(s,a)| \lesssim \twonorm{\cov^{-1/2} \phi(s,a)}(\epsfp+ \frac{1}{\smin(I -\whitecross)}  \levcov  \epsinf ) + \epsinf.
\end{align}
\end{proposition}

The main message of this proposition, is that if linear realizability fails, but invertibility still holds, then the performance of LSTD and other linear estimators degrades gracefully with the level of misspecification. 

To help parse the result, we can walk through each of the terms appearing on the right hand side of \Cref{eq:qpi_miss}.  
The first source of error, captured in $\epsfp$, is statistical in nature. 
It arises from bounding the statistical error inherent in estimating $\Qpi$ by approximating the fixed point solution to the (projected) Bellman equation, $\tfp$.
Note that controlling this term is the precisely the main focus on the previous results upper bounding the error of LSTD and FQI. 

Because $\tsg$, as defined in \Cref{eq:bellman_fixed_point}, equals $\tfp$, 
if invertibility holds, then for large enough $n$, \Cref{theorem:lstd_upper_bound} proves that 
LSTD return a vector such that, with probability $1-\delta$,
\begin{align*}
	\epsfp \lesssim \frac{1}{\smin(I - \gamma \cov^{-1/2}\cross\cov^{-1/2})} \cdot \epsz + \frac{1}{\smin(I - \gamma \cov^{-1/2}\cross\cov^{-1/2})^2} \twonorm{\cov^{-1/2} \theta_{\phi,r}}\cdot \epsop
\end{align*}

Likewise, \Cref{theorem:main_result} shows that if stability holds, then for large enough $n$, performing $T$-steps of FQI return a solution $\thetahat_T$ such that with probability $1-\delta$,
\begin{align*}
	\epsfp &\lesssim \cond(P_\gamma)^{1/2}\opnorm{P_\gamma} \cdot \epsz + \cond(P_\gamma)\opnorm{P_\gamma}^2 \cdot \twonorm{\cov^{-1/2}\theta_{\phi,r}} \cdot  \epsop \notag + \mathcal{O}(\exp(-T)).
\end{align*}
As we might expect, this statistical error, $\epsfp$ becomes vanishingly small as the number of samples goes to infinity, regardless of whether $\Qpi$ is linearly realizable.

The second set of terms in \Cref{prop:misspecification}, depending on $\epsinf$, come from the fact that $\Qpi$ cannot be expressed as a linear function of the features $\phi$. 
 Consequently, this term does \emph{not} go to zero as the number of samples becomes large. This approximation error is amplified by a factor of $\levcov  / \smin(I - \whitecross)$.
 Since this is only an upper bound, we cannot assert that these multiplicative factors are necessary. 
 However, the dependence on the statistical leverage $
 \levcov$ is reminiscent of previous upper bounds from the linear bandits literature \citep{lattimore2020learning} where the approximation error is also amplified by a factor of $\sqrt{d}$.\footnote{The statistical leverage $\levcov$ is exactly $\sqrt{d}$ in the best case.} \cite{du2019good} and \cite{van2019comments} provide similar lower bounds under approximate misspecification of the relevant feature mappings.

In any case, beyond the specific scaling on the various error sources, the main take away message from this result is that FQI and LSTD are reasonable estimators to use beyond the linear realizability setting. Under the necessary assumption that invertibility (or stability) hold, the extent to which these methods estimate the underlying value functions is only mildy affected by the approximation error $\epsinf$. As alluded to previously, the results in this section are not the focus of our work. We primarily view them  as a first step towards a more complete understanding of offline policy evaluation in the absence of realizability.

\section{Discussion}
\label{sec:discussion}

In this work, we characterize the exact limits of linear estimators for offline policy evaluation, under the assumption that the value function is linearly realizable in some known set of features. 
Our stability and invertibility based analyses introduce new, sharper notions of complexity for this classical setting and provide a simple, unifying perspective which brings together previously disparate analysis of popular algorithms. 

Two extensions to our results pertain to the finite horizon setting and to policy optimization. As a starting point, we have focused on the infinite horizon, discounted setting as the conditions there are cleaner than in the finite horizon case. Nevertheless, we conjecture that Lyapunov stability and invertibility can be used to analyze finite horizon problems as well. 
Regarding policy optimization, understanding when this task is possible under linear realizability is an important direction for future work. We hope that our characterization of linear estimators for policy evaluation provides a useful perspective on this closely related problem. 

Apart from these extensions, it would be valuable to study quantitative, instance-dependent lower bounds on the sample complexity necessary for offline policy evaluation under linear realizability. In particular, our characterization of linear estimators is sharp in the sense that we precisely determine when the value function of a policy is \emph{identifiable} (alternatively, \emph{learnable}) using classical methods. Having established that a problem is learnable, it is interesting to understand whether the estimation rates for the various algorithms are sharp in a worst case or instance dependent sense.  

\section{Acknowledgments}

We gratefully acknowledge the support of Microsoft through the BAIR Open Research Commons. JCP was in part supported by an NSF Graduate Research Fellowship. 
Sham Kakade acknowledges funding from the National
Science Foundation under award $\#$CCF-1703574 and the Office of Naval Research under
award N00014-21-1-2822. We would also like to thank the anonymous reviewers whose insightful comments greatly improved the resulting manuscript.

\bibliographystyle{plainnat}
\newpage
\bibliography{refs}
\newpage
\appendix


\section{Supporting Arguments for \Cref{sec:main_results}: FQI}

\subsection{Proof of \Cref{theorem:main_result}: stability is sufficient for FQI}
\label{subsec:proof_main_theorem}
The existence of $\covhat$ and the upper bounds on the regression errors $\epsz$ and $\epsop$ are guaranteed by \Cref{lemma:epsop_concentration,lemma:epsz_concentration}.
To analyze the error of FQI, we introduce the shorthand,
\begin{align*}
A \defeq \gamma\cov^{-1}\cross, \quad \Ahat = \gamma \covhat^{-1} \crosshat, \quad
 	\theta_t^\star  \defeq \sum_{k=0}^t A^k \theta_0^\star, \quad \widehat{\theta}_t \defeq \sum_{k=0}^t \Ahat^k \widehat{\theta}_0,
\end{align*}
\begin{align*}	
 	\quad w_t \defeq \widehat{\theta}_t - \theta_t^\star, \quad \Delta \defeq \Ahat - A,
\end{align*}
where $\widehat{\theta}_0 = \covhat^{-1} \widehat{\theta}_{\phi,r}$ and $\theta_0^\star = \cov^{-1} \theta_{\phi,r}$. Using this notation, by stability, we observe that $\theta_\gamma^\star = \theta_\infty^\star,$ and we can write the errors vectors of the $t$-step  FQI solution as,  
\begin{align}
\label{eq:exact_error_vectors}
\cov^{1/2}(\theta_\gamma^\star - \widehat{\theta}_t) = \cov^{1/2} \sum_{k=t+1}^\infty A^k \theta_0^\star  + \cov^{1/2}w_t.
\end{align}
Next, we develop the recursion in $w_t$, 
\begin{align*}
 w_{t+1} &= \sum_{j=0}^{t+1} \Ahat^j \widehat{\theta}_0 - \sum_{j=0}^{t+1} A^j \theta_0^\star \\ 
 & = \widehat{A} \widehat{\theta}_{t} + \widehat{\theta}_0 - A \theta_t^\star - \theta_0^\star \\ 
 & = \Ahat w_t + \Delta \theta_t^\star + w_0.
\end{align*}
Unrolling the recursion and multiplying on the left by $\cov^{1/2}$, we get that 
\begin{align*}
	\cov^{1/2} w_{t+1} = \sum_{j=0}^{t+1}\left( \cov^{1/2} \Ahat \cov^{-1/2} \right)^{j} \cov^{1/2} w_0 + \sum_{j=0}^t \left( \cov^{1/2}\Ahat \cov^{-1/2}\right)^j \left(\cov^{1/2}  \Delta \cov^{-1/2} \right) \cov^{1/2}\theta_{t-j}^\star.
\end{align*}
Note that $\epsz = \twonorm{\cov^{1/2} w_0}$ and $\epsop = \opnorm{\cov^{1/2}  \Delta \cov^{-1/2}}$. Therefore, taking the norm of both sides and applying the triangle inequality,
\begin{align}
	\twonorm{\cov^{1/2} w_{t+1}} & \leq  \sum_{k=0}^{t+1} \opnorm{\left( \cov^{1/2} \Ahat \cov^{-1/2} \right)^{k}} \twonorm{\cov^{1/2} w_0} \\
	&+ \sum_{k=0}^t \opnorm{\left( \cov^{1/2}\Ahat \cov^{-1/2}\right)^k} \opnorm{\cov^{1/2}  \Delta \cov^{-1/2} } \sup_{0\leq h\leq t} \twonorm{ \cov^{1/2}\theta_{h}^\star} \notag\\ 
	& =  ( \epsz + \epsop \sup_{0\leq h\leq t} \twonorm{ \cov^{1/2}\theta_{h}^\star})  \cdot  \sum_{k=0}^{t+1} \opnorm{\left( \cov^{1/2} \Ahat \cov^{-1/2} \right)^{k}}.  \label{eq:wt_bound}
\end{align}
Now, recalling the definition of $\theta_h^\star$, we bound:
\begin{align}
\label{eq:sup_bound}
	\sup_{0 \leq h\leq t} \twonorm{\cov^{1/2}\theta_h^\star} \leq \sum_{j=0}^t \opnorm{\left(\gamma \cov^{-1/2} \cross \cov^{-1/2} \right)^j } \twonorm{\cov^{1/2}\theta_0^\star}.
\end{align}
Therefore, combining these last two inequalities \eqref{eq:sup_bound}, \eqref{eq:wt_bound}, and the identity from \Cref{eq:exact_error_vectors},
\begin{align}
\twonorm{\cov^{1/2}(\theta_\gamma^\star - \widehat{\theta}_t)} &\leq  \sum_{k=t+1}^\infty \opnorm{\left(\cov^{1/2} A \cov^{-1/2}\right)^k} \twonorm{\cov^{1/2} \theta_0^\star}  + \twonorm{\cov^{1/2}w_t} \notag \\
& \leq \twonorm{\cov^{1/2} \theta_0^\star} \sum_{k=t+1}^\infty \alpha_k + \left(\epsz + \epsop \twonorm{\cov^{1/2} \theta_0^\star}  \sum_{k=0}^{t-1} \alpha_k\right) \sum_{k=0}^t \widehat{\alpha}_k, \label{eq:key_inequality_mainthm}
\end{align}
where $\widehat{\alpha}_k \defeq \opnorm{\left( \cov^{1/2} \Ahat \cov^{-1/2} \right)^{k}}$ and $\alpha_k \defeq \opnorm{\left( \cov^{1/2} A \cov^{-1/2} \right)^{k}}$.
Since $\epsop \leq 1 / (6 \opnorm{P_\gamma})$ and $\whitecross $ is stable, \Cref{lemma:stability_margin} tells us that 
\begin{align*}
	\widehat{\alpha}_j  & =\opnorm{P_\gamma^{-1/2}P_\gamma^{1/2}\left(\cov^{1/2}\Ahat \cov^{-1/2} \right)^j} \\
	&\leq \opnorm{P_\gamma^{-1/2}} \opnorm{P_\gamma^{1/2} \left(\cov^{1/2}\Ahat \cov^{-1/2} \right)^j} \\
	&\leq \opnorm{P_\gamma^{1/2}} \opnorm{P_\gamma^{1/2}}\left(1 - \frac{1}{2 \opnorm{P_\gamma}} \right)^{j/2} = \cond(P_\gamma)^{1/2} \left(1 - \frac{1}{2 \opnorm{P_\gamma}} \right)^{j/2}.
\end{align*}
Using similar reasoning, we get that 
\begin{align*}
	\alpha_j \leq \cond(P_\gamma)^{1/2} \left( 1 - \frac{1}{\opnorm{P_\gamma}}\right)^{j/2}.
\end{align*}
In conclusion, $\twonorm{\cov^{1/2}(\theta_\gamma^\star - \widehat{\theta}_t)}$ is bounded by, 
\begin{align*}
\twonorm{\cov^{1/2} \theta_0^\star}  \cond(P_\gamma)^{1/2}\left( 1 - \frac{1}{\opnorm{P_\gamma}}\right)^{(t+1)/2}\sum_{k=0}^\infty \alpha_k + \left(\epsz + \epsop \twonorm{\cov^{1/2} \theta_0^\star}  \sum_{k=0}^{\infty} \alpha_k\right) \sum_{k=0}^{\infty} \widehat{\alpha}_k.
\end{align*}
The final bound comes from summing the geometric series, $\sum_{j=0}^\infty (1 - c)^{j/2} = (1 - \sqrt{1-c})^{-1}$, for $c \in (0,1)$ and applying the numerical inequality,
\begin{align*}
\left(1 - \sqrt{1 - \frac{1}{2z}}\right)^{-1} \leq 10z,
\end{align*}
which holds for all $z \geq 1$. 
\begin{lemma}
\label{lemma:stability_margin}
Let $A$ be a square, stable matrix and let $P = \dlyaps{A}$ Then, for all $k\geq 0$, 
\begin{align*}
	\opnorm{A^k}^2 \leq \cond(P)\left(1 - \frac{1}{\opnorm{P}}\right)^{k}.
\end{align*}
Furthermore, for any matrix $\Delta$ such that $\opnorm{\Delta} \leq 1 / (6 \opnorm{P}^2)$, 
\begin{align*}
	\opnorm{(A+\Delta)^k}^2 \leq \cond(P) \left(1-\frac{1}{2\opnorm{P}}\right)^k.
\end{align*}
\end{lemma}

\begin{proof}
This particular lemma is almost identical to the one from \cite{perdomo2021stabilizing}. However, we include the proof for the sake of providing a self-contained presentation. For the first result, by definition of the solution to the Lyapunov equation, for any unit vector $x$,
\begin{align*}
	x^\top A^\top P Ax &= x^\top P x - x^\top I x \\
	&= x^\top P x \left(1 - \frac{\twonorm{x}^2}{x^\top Px } \right) \\
	& \leq x^\top P x \left(1 - \frac{1}{\opnorm{P}} \right).
\end{align*}
Hence, $A^\top P A \preceq P(1 - \opnorm{P}^{-1})$. By iterating $(A^k)^\top P A^k \preceq P(1 - \opnorm{P}^{-1})^k$ and 
\begin{align*}
\opnorm{P^{1/2}A^k}^2 \leq \opnorm{P} (1 - \opnorm{P}^{-1})^k.
\end{align*}
Therefore, 
\begin{align*}
	\opnorm{A^k} = \opnorm{P^{-1/2}P^{1/2}A^k}  \leq \opnorm{P^{-1/2}} \opnorm{P^{1/2}A^k} \leq \cond(P)^{1/2}\left(1 - \frac{1}{\opnorm{P}}\right)^{k/2}.
\end{align*}
For the second result, using the insights from above,
\begin{align*}
	(A+\Delta)^\top P (A + \Delta) &= A^\top P A + A^\top P \Delta + \Delta^\top P A + \Delta^\top P \Delta.
\end{align*}
Now, $	A^\top P A \preceq P (1 - \opnorm{P}^{-1})$ and 
\begin{align*}
	\opnorm{A^\top P \Delta} = \opnorm{\Delta^\top P A}\leq \opnorm{\Delta P^{1/2}}\opnorm{P^{1/2}A} \leq \opnorm{\Delta P^{1/2}} \opnorm{P^{1/2}} \leq \opnorm{\Delta} \opnorm{P}.
\end{align*}
Bounding, $\opnorm{\Delta^\top P \Delta} \leq \opnorm{P}\opnorm{\Delta}^2$, and using the fact that $P \succeq I$ we get that for, $$\opnorm{\Delta}\leq 1 / (6 \opnorm{P}^2),$$
the following relationship holds:
\begin{align*}
	A^\top P \Delta + \Delta^\top P A + \Delta^\top P \Delta \preceq P \frac{1}{2\opnorm{P}}.
\end{align*}
Therefore, 
\begin{align*}
(A+\Delta)^\top P (A + \Delta) \preceq P \left(1 - \frac{1}{2\opnorm{P}}\right),
\end{align*}
and the second result follows by using the same steps as the first.
\end{proof}

\subsection{Proof of \Cref{prop:coordinate_free}: coordinate invariance of $P_\gamma$}

If we define the whitened features, $\wphi(\cdot) = \cov^{-1/2} \phi(\cdot)$, then $\tphi(\cdot) = L' \wphi(\cdot)$ where $L' = L \cov^{1/2}$. Now, let $USV^\top$ be the singular value decomposition of $L'$. Then, 
\begin{align*}
	\tcov = \E_{x \sim \calD} \tphi(x)\tphi(x)^\top  = L' \E_{x \sim \calD} \wphi(x) \wphi(x)^\top L'^\top = L' L'^\top = US^2U^\top,
\end{align*}
where we have used the fact that the whitened features have identity covariance. By this calculation, we have that $\tcov^{1/2} = USU^\top$. Using similar substitutions, we can also deduce that $\tcross = L' \cross^{(w)} L'^\top$ where $\cross^{(w)} = \cov^{-1/2} \cross \cov^{-1/2}$.  Therefore, 
\begin{align*}
	\tcov^{-1/2}\tcross \tcov^{-1/2} = (US^{-1} U^\top)(USV^\top) \cross^{(w)}(VSU^\top)(US^{-1}U^\top) = (UV^\top) \cross^{(w)}(UV^\top)^\top. 
\end{align*}
Since $(UV^\top)$ is an orthogonal matrix, the equality of condition numbers follows by the fact that for any matrix $A$ and orthogonal matrix $M$, $MAM^\top = A$ have the same singular values. On the other hand, the invariance of the operator norm of $P_\gamma$ follows from the following lemma:

\begin{lemma}
Let $A$ be a stable matrix and $M$ be any orthogonal matrix, then 
\begin{align*}
\opnorm{\dlyaps{A^\top}} = \opnorm{\dlyaps{M A^\top M^\top}}.
\end{align*}
\end{lemma}
\begin{proof}
Let $P = \dlyaps{A}$ be the unique solution over X to the matrix equation: 
\begin{align*}
	X = A^\top X A + I.
\end{align*}
Likewise, let $P' = \dlyaps{MAM^\top}$ be the unique solution (over $X'$) to the equation:
\begin{align*}
	X' = MA^\top M^\top X' MAM^\top + I.
\end{align*}
From this, we can deduce that $M^\top X'M = A^\top M^\top X' M A + I$. Therefore, $P  = M^\top X'M = M^\top P'M$. The conclusion  follows from the fact that singular values are invariant to conjugation by an orthogonal matrix.
\end{proof}

\subsection{Proof of \Cref{prop:contractive_case}: FQI under specific growth rates}

As discussed in the main body, the proof is identical to that of \Cref{theorem:main_result} except that we specialize to the particular assumptions on the growth of matrix powers. We recall the key inequality from the proof of the main theorem, \Cref{eq:key_inequality_mainthm}:
\begin{align*}
\twonorm{\cov^{1/2}(\widehat{\theta}_t - \theta_\gamma^\star)}\leq \twonorm{\cov^{1/2} \theta_0^\star} \sum_{k=t+1}^\infty \alpha_k + \left(\epsz + \epsop \twonorm{\cov^{1/2} \theta_0^\star}  \sum_{k=0}^{t-1} \alpha_k\right) \sum_{k=0}^t \widehat{\alpha}_k.
\end{align*}
Here, $\alpha_k = \opnorm{\left(\whitecross\right)^k}$ and $\widehat{\alpha}_k \defeq \opnorm{\left( \cov^{1/2} (\gamma \covhat^{-1} \crosshat) \cov^{-1/2} \right)^{k}}$. By assumption, $\alpha_k \leq \alpha \beta^k$ hence, $\sum_{k=0}^\infty \alpha_k \leq \alpha / (1 - \beta)$.
Now, by \Cref{lemma:binomial_expansion} since $$\epsop = \opnorm{ \cov^{1/2} (\gamma \covhat^{-1} \crosshat) \cov^{-1/2} - \whitecross},$$ 
we have that:
\begin{align*}
\widehat{\alpha}_k \leq \alpha (\beta + \epsop\alpha)^k.
\end{align*}
Therefore, as long as $\epsop <  \frac{9}{10} \frac{(1 - \beta)}{\alpha} $, 
\begin{align*}
	\sum_{k=0}^\infty \widehat{\alpha}_k \leq \alpha \sum_{k=0}^\infty (\beta + \epsop \alpha)^k = \alpha\frac{1}{1 -\beta - \alpha \epsop} \leq 10  \frac{\alpha}{1-\beta}.
\end{align*}
Putting everything together, 
\begin{align*}
\twonorm{\cov^{1/2}(\widehat{\theta}_T - \theta_\gamma^\star)}\lesssim \twonorm{\cov^{1/2} \theta_0^\star}  \frac{\alpha}{1-\beta} \cdot \beta^{T+1}+ \left(\epsz + \epsop \twonorm{\cov^{1/2} \theta_0^\star}  \frac{\alpha}{1-\beta} \right) \frac{\alpha}{1-\beta}.
\end{align*}

\begin{lemma}
\label{lemma:binomial_expansion}
Let $A$ be a square matrix such that for all nonnegative integers $j$, $\opnorm{A^j} \leq a \cdot b^j$ for scalars $a > 0$ and $b \in(0,1)$. Then, for any square matrix $\Delta$ if we let $\eps \defeq \opnorm{\Delta}$ then,
 \begin{align*}
	\opnorm{(A + \Delta)^n} \leq a (b + \epsilon \cdot a)^n.
\end{align*}
\end{lemma}
\begin{proof}
We begin by expanding $(A + \Delta)^n$ into monomials $T_{k,j}$, 
\begin{align}
\label{eq:monomials}
	(A + \Delta)^n = \sum_{k=0}^{n} \sum_{j=1}^{\binom{n}{k}} T_{k,j},
\end{align}
where each $T_{k,j}$ has $k$ factors of $\Delta$ and $n-k$,  $A$ factors. Now, by the submultiplicative property of the operator norm, 
\begin{align*}
	\opnorm{T_{k,j}} \leq \epsilon^k \prod_{s_i \in S_{k,j}} \opnorm{A^{s_i}},
\end{align*}
where $S_{k,j}$ is a set of positive integers $s_i$ satisfying $\sum_i s_i = n-k$ and $|S| \leq k+1$. Using our assumption on the growth of $\opnorm{A^k}$, we get that,
\begin{align*}
	\opnorm{T_{k,j}} &\leq  \epsilon^k 
\prod_{s_i \in S_{k,j}} (a \cdot b^{s_i}) \leq a^{k+1} \epsilon^k  b^{n-k}.
\end{align*}
Going back to the original expansion into monomials, and using the identity, 
\begin{align*}
\sum_{k=0}^n \binom{n}{k} x^k = (1 + x)^n.
\end{align*}
We conclude:
\begin{align*}
	\opnorm{(A+\Delta)^n} &\leq a \cdot b^n \sum_{k=0}^n \binom{n}{k} \left(\frac{a \epsilon}{b}\right)^k = a b^n (1 +\frac{a \cdot \epsilon}{b} )^n = a (b + a\epsilon)^n.
\end{align*}
\end{proof}

\subsection{Proof of \Cref{corollary:low_distribution_shift}: low distribution shift implies stability}

Consider the augmented covariance matrix,
\begin{align*}
	\E \begin{bmatrix}
		\phi(s,a) \\ 
	\phi(s', a')
	\end{bmatrix}
	\begin{bmatrix}
		\phi(s,a) \\ 
	\phi(s', a')
	\end{bmatrix}^\top  = \begin{bmatrix}
		\cov & \cross \\ 
		\cross^\top & \nextcov
	\end{bmatrix} \succeq 0.
\end{align*}
By a Schur complement argument, $\cross^\top  \cov^{-1} \cross \preceq \nextcov$. After conjugating by $\cov^{-1/2}$ and multiplying by $\gamma^2$, we get that:
\begin{align*}
 (\gamma \cov^{-1/2} \cross \cov^{-1/2})^\top ( \gamma \cov^{-1/2} \cross \cov^{-1/2}) \preceq \gamma^{2} \cov^{-1/2} \nextcov \cov^{-1/2}. 
\end{align*}
Now, by the low distribution shift assumption, 
$\gamma^{2} \cov^{-1/2} \nextcov \cov^{-1/2} \preceq \gamma^{2} \cov^{-1/2} (\cds \cov) \cov^{-1/2} = \cds \gamma^2 I.$
Therefore, $(\gamma \cov^{-1/2} \cross \cov^{-1/2})^\top ( \gamma \cov^{-1/2} \cross \cov^{-1/2}) \preceq \Cds \gamma^2I$. Iterating for $j\geq0$ gives the first part of the result. The rest follows from \Cref{prop:contractive_case} by observing that \Cref{eq:growthrates} holds with $\alpha=1, \beta = \sqrt{\cds \gamma^2} \in (0,1)$.

\subsection{Proofs of \Cref{corollary:completeness_case}: Bellman completeness implies stability}

To take advantage of matrix notation, for this result we assume that the state-action space is finite, $|\states||\actions| < \infty$. In particular, we introduce the following quantities. 
\begin{enumerate}
	\item Feature matrix $\Phi \in \R^{|\states||\actions| \times d}$.
	\item Offline distribution vector $\mu \in \R^{|\states||\actions|}$.
\end{enumerate}
With this, we have that $\cov = \Phi^\top \diag(\mu) \Phi$ and $\cross = \Phi^\top \diag(\mu) \trop \Phi$ where $P^{(\pi)}$ is a row stochastic matrix representing the transition operator. \Cref{corollary:completeness_case} follows from the following lemma and \Cref{prop:contractive_case}.

\begin{lemma}
\label{lemma:completeness_power_lemma}
If $\phi$ is complete (\Cref{ass:completeness}) and $\cov$ is full rank, then for $j \geq 0$, 
\begin{align*}
\opnorm{(\cov^{-1/2} \cross \cov^{-1/2})^j} \leq \levcov.	
\end{align*}
\end{lemma}

\begin{proof}
First, we rewrite the relevant matrix as follows, 
\begin{align*}
	(\cov^{-1/2} \cross \cov^{-1/2})^j &= \cov^{1/2} (\cov^{-1} \cross)^j \cov^{-1/2} \\ 
	& = \cov^{-1/2} \Phi^\top \diag(\mu) \Phi^\top (\cov^{-1} \cross)^j
 \cov^{-1/2}.
\end{align*}
Therefore, 
\begin{align*}
	\opnorm{(\cov^{-1/2} \cross \cov^{-1/2})^j} \leq \underbrace{\opnorm{\cov^{-1/2} \Phi^\top \diag(\mu)^{1/2}}}_{\defeq T_1} \underbrace{\opnorm{\diag(\mu)^{1/2} \Phi (\cov^{-1} \cross)^j \cov^{-1/2}}}_{\defeq T_2}.
\end{align*}
To bound $T_1$, we observe that 
\begin{align*}
	\opnorm{\cov^{-1/2} \Phi^\top \diag(\mu)^{1/2}}^2 = \opnorm{(\Phi^\top \diag(\mu) \Phi)^{-1/2} \Phi^\top \diag(\mu)^{1/2}}.
\end{align*}
Letting $A \defeq \diag(\mu)^{1/2} \Phi$, the above expression satisfies,
\begin{align*}
	\opnorm{(A^\top A)^{-1/2} A^\top}^2 = \sup_{\twonorm{v} = 1} v^\top A(A^\top A)^{-1} A^\top v \leq 1,
\end{align*}
since $A(A^\top A)^{-1} A^\top$ is a projection matrix. Moving onto $T_2$, we recall that 	
\begin{align*}
\opnorm{\diag(\mu)^{1/2} \Phi (\cov^{-1} \cross)^j \cov^{-1/2}} = \sup_{\twonorm{v} =1} \twonorm{\diag(\mu)^{1/2} \Phi (\cov^{-1} \cross)^j \cov^{-1/2} v}.
\end{align*}
For any fixed vector $v$, since the entries of $\mu$ form a probability measure,
\begin{align*}
	\twonorm{\diag(\mu)^{1/2} v} = \sqrt{\sum_{i=1}^d \mu_i v_i^2} \leq \max_i v_i = \infnorm{v}.
\end{align*}
Therefore, 
\begin{align*}
	\opnorm{\diag(\mu)^{1/2} \Phi (\cov^{-1} \cross)^j \cov^{-1/2}} \leq \sup_{\twonorm{v} =1} \infnorm{\Phi (\cov^{-1} \cross)^j \cov^{-1/2} v}.
\end{align*}
Then, by repeatedly applying \Cref{lemma:completeness_helper}, we get that
\begin{align*}
	\infnorm{\Phi (\cov^{-1} \cross)^j \cov^{-1/2} v} \leq \infnorm{\Phi \cov^{-1/2} v}.
\end{align*}
Lastly, 
\begin{align*}
	\infnorm{\Phi \cov^{-1/2} v} = \sup_{(s,a)} | \phi(s,a)^\top \cov^{-1/2} v | \leq \sup_{(s,a)\in \states \times \actions} \twonorm{\cov^{-1/2} \phi(s,a)} = \levcov.
\end{align*}
\end{proof}

\begin{lemma}
\label{lemma:completeness_helper}
If $\phi$ is complete (\Cref{ass:completeness}) and $\cov$ is full rank, then for all $\theta$, 
\begin{align*}
\infnorm{\Phi \cov^{-1} \cross\theta} \leq \infnorm{\Phi \theta}.	
\end{align*}
\end{lemma}
\begin{proof}
If we denote the vector of expected rewards by $\vec{r} \in \R^{|\states||\actions|}$, then completeness implies that for all $\theta$, there exists a $\theta'$ such that 
\begin{align*}
	\Phi \theta' = \vec{r} + \gamma \trop \Phi \theta.
\end{align*}
Choosing $\theta=0$, this means that there exists a vector $\theta_r$ such that $\vec{r} = \Phi \theta_r$. Consequently, we deduce that for all $\theta$, there always exists a $\theta'$ such that $\Phi \theta' = \gamma \trop \Phi \theta$. Using this realizability condition, for a given distribution $\mu$, $\theta'$ must satisfy 
\begin{align*}
	\theta' &= \argmin_{\bar{\theta}} \E_{(s,a)\sim \mu, s' \sim P(\cdot \mid s,a)} \left(\phi(s,a)^\top \bar{\theta} - \gamma  \phi(s',a')^\top \theta \right)^2 \\
	&=  \gamma \cov^{-1} \cross \theta.
\end{align*}
Together with the previous equation, this implies that for all $\theta$, $\gamma \Phi  \cov^{-1} \cross \theta
=\gamma \trop \Phi \theta$. Thus, we conclude that 
\begin{align*}
	\infnorm{\Phi \cov^{-1} \cross\theta} &= \infnorm{ \trop \Phi \theta} \\ 
	& \leq \infnorm{\Phi \theta},
\end{align*}
where we have used the fact that $\trop$ is row stochastic so $\| \trop\|_{1} \leq 1$.
\end{proof}

\begin{lemma}
\label{lemma:reward_realizability}
Assume that the rewards are linearly realizable in the feature mapping $\phi$. That is, there exists a vector $\theta_r^\star \in \R^d$ such that for all $(s,a)\in \states \times \actions$, $\E r(s,a) = \phi(s,a)^\top \theta_r^\star$. Then, $\twonorm{\cov^{-1/2}\theta_{\phi,r}} \leq 1$.

Otherwise, if reward realizability does not hold $ \twonorm{\cov^{-1/2} \theta_{\phi,r}} \leq \sqrt{d}$.
\end{lemma}

\begin{proof}
Expanding out the definition of $\theta_{\phi,r}$, 
\begin{align*}
	\twonorm{\cov^{-1/2}\theta_{\phi,r}}^2 = \traceb{\cov^{-1/2} \E[\phi(s,a)r(s,a)] \E \phi(s,a)^\top r(s,a) \cov^{-1/2}} 
\end{align*}
Under realizability, $\E[\phi(s,a)r(s,a)] = \E \phi(s,a)\phi(s,a)^\top \theta_r^\star$. Hence, the expression above can be rewritten as, 
\begin{align*}
	 \traceb{\cov^{-1} \E[\phi(s,a)\phi(s,a)^\top] \theta_r^\star\theta_r^{\star\top} \E \phi(s,a) \phi(s,a)^\top } = \E (\phi(s,a)^\top \theta_r^\star)^2  = \E r(s,a)^2 \leq 1. 
\end{align*}
If the rewards are not linearly realizable in $\phi$, then by Jensen's inequality,
\begin{align*}
\twonorm{\cov^{-1/2} \E \phi(s,a) r(s,a)}^2 &\leq \E \twonorm{\cov^{-1/2}  \phi(s,a)r(s,a)}^2 \\
&= \E \traceb{\cov^{-1/2} \E \phi(s,a)\phi(s,a)^\top r(s,a)^2 \cov^{-1/2}} \\
&\leq  \sup_{s,a} r(s,a)^2 \traceb{I} \\
& \leq d.
\end{align*}
\end{proof}

\subsection{Proof of \Cref{prop:fqi_lb}: FQI lower bound}

Recall the functional form of the FQI approximation,
\begin{align*}
\widehat{\theta}_T = \sum_{k=0}^T (\gamma \cov^{-1} \cross)^{k} \cov^{-1}(\theta_{\phi,r} + z) = \mu + v,
\end{align*}
where $\E \widehat{\theta}_T = \mu \defeq \sum_{k=0}^T (\gamma \cov^{-1} \cross)^{-1} \cov^{-1}\theta_{\phi,r}$ and $v \defeq \sum_{k=0}^T (\gamma \cov^{-1} \cross)^{-1} \cov^{-1}z$. Expanding out and using $\E v = 0$, we have that 
\begin{align*}
	\E \twonorm{\widehat{\theta}_T - \E \widehat{\theta}_T}^2 &= \E \twonorm{\widehat{\theta}_T}^2 - \twonorm{\E \widehat{\theta}_T}^2  \\
	& = \E \twonorm{\mu}^2 + \E \twonorm{v}^2 - \twonorm{\E \widehat{\theta}_T}^2\\
	& = \E \twonorm{v}^2.
\end{align*}
Now, letting $A = \gamma \cov^{-1} \cross$, we have that
\begin{align*}
	\E\twonorm{v}^2 &= \traceb{(\sum_{k=0}^T A^k)^\top \Lambda (\sum_{k=0}^T A^k)} 
	\geq \smin(\Lambda) \opnorm{\sum_{k=0}^T A^k}^2
	= \smin(\Lambda) \sup_{\twonorm{v}=1} v^\top (\sum_{k=0}^T A^k)^\top (\sum_{k=0}^T A^k)v,
\end{align*}
where we have used $\traceb{A^\top A} = \fnorm{A}^2 \geq \opnorm{A}^2$ ($\fnorm{\cdot}$ denotes the Frobenius norm of  a matrix) and the variational characterization of the operator norm for symmetric matrices. By assumption on the spectral radius, $A$ has an eigenvector $u$ with eigenvalue $\lambda$ such that $|\lambda| > 1$. Therefore,
\begin{align*}
\sup_{\twonorm{v}=1} v^\top (\sum_{k=0}^T A^k)^\top (\sum_{k=0}^T A^k)v \geq u^\top (\sum_{k=0}^T A^k)^\top (\sum_{k=0}^T A^k)u = \twonorm{u}^2 (\sum_{k=0}^T \lambda^k )^2 = \left(\frac{\lambda^{T+1} - 1}{\lambda -1}\right)^2.
\end{align*}
Note that if $|\lambda|=1$, this series can grow linearly in $T$ (e.g if $\lambda=1$) or oscillate (if $\lambda=-1$). The last equality only holds for $\lambda \neq 1$.

\subsection{Extensions to ridge regression}
\label{subsec:ridge}
One might wonder whether adding $\ell_2$ regularization, that is, an $\lambda \twonorm{\theta}^2, \lambda > 0$ additive penalty to the FQI or LSTD objective in \Cref{eq:fqi_def}, could help mitigate the divergence phenomenon outlined in \Cref{prop:fqi_lb} or the limits of linear estimators from \Cref{theorem:lower_bound_lspe}.

For finite-dimensional problems with full rank covariance, typical analyses of ridge regression set the regularizer $\lambda$ to shrink with the number of samples $n$. In this case, the ridge estimator achieves consistent parameter recovery and asymptotically returns the same solution as just performing ordinary least squares. Therefore, we can expect similar blowup if stability fails  (in fact, this phenomenon is verified empirically by \cite{wang2021Instabilities}). On the other hand, if the parameter $\lambda$ is lower bounded by a constant, then ridge regression will have constant bias which will then be amplified by the number of rounds $T$. Hence, adding regularization does not avoid the need for stability when performing fitted Q-iteration. 
Similar arguments demonstrate why regularization is unlikely to overcomes the limitations of least squares temporal differencing learning (or other linear estimators) in settings where invertibility does not hold.

\section{Supporting Arguments for \Cref{sec:lspe}: LSTD}

\subsection{Proof of \Cref{theorem:lstd_upper_bound}: invertibility is sufficient for LSTD}

Recall the closed form expression of the empirical LSTD estimator: 
\begin{align*}
\thetahatls = (I - \gamma\covhat^{-1}\crosshat)^{\dagger}  \covhat^{-1} \widehat{\theta}_{\phi,r}.
\end{align*}
Multiplying on the left by $\cov^{1/2}$, 
\begin{align*}
	\cov^{1/2} \thetahatls &= \cov^{1/2} (I - \gamma\covhat^{-1}\crosshat)^{\dagger} \cov^{-1/2} \cov^{1/2} \covhat^{-1} \widehat{\theta}_{\phi,r} \\
	& =\left(\cov^{1/2} (I - \gamma\covhat^{-1}\crosshat) \cov^{-1/2} \right)^{\dagger}(\cov^{1/2} \covhat^{-1} \widehat{\theta}_{\phi,r}),
\end{align*}
where we have used the identity $(ABA^{-1})^{\dagger} = AB^{\dagger}A^{-1}$ for any invertible $A$ and $B$. Similarly, 
\begin{align*}
	\cov^{1/2}\theta^\star_\gamma = \left( \cov^{1/2}(I-\gamma \cov^{-1}\cross)\cov^{-1/2} \right)^{-1}(\cov^{1/2} \cov^{-1} \theta_{\phi,r}).
\end{align*}
Now defining the following quantities, 
\begin{align*}
	A \defeq \cov^{1/2}(I-\gamma \cov^{-1}\cross)\cov^{-1/2},\quad \Ahat \defeq \cov^{1/2} (I - \gamma\covhat^{-1}\crosshat) \cov^{-1/2} 
\end{align*}
\begin{align*}
	b \defeq \cov^{1/2} \cov^{-1} \theta_{\phi,r}, \quad \widehat{b} \defeq \cov^{1/2} \covhat^{-1} \widehat{\theta}_{\phi,r}.
\end{align*}
We can rewrite the above expression as:
\begin{align*}
	\cov^{1/2}(\theta_\gamma^\star - \widehat{\theta}_\gamma) = (A^{-1} - \Ahat^{\dagger})b + \Ahat^{\dagger}(b-\widehat{b}).
\end{align*}
Therefore, 
\begin{align*}
	\twonorm{\cov^{1/2}(\theta_\gamma^\star - \widehat{\theta}_\gamma)} \leq \opnorm{A^{-1} - \Ahat^{\dagger}} \twonorm{b} + \opnorm{\Ahat^{\dagger}} \twonorm{b - \widehat{b}}.
\end{align*}
Using \Cref{lemma:matrix_perturbation}, since $\epsop \leq \frac{1}{2} \smin(I - \gamma \cov^{-1}\cross)$:
\begin{align*}
	\twonorm{\cov^{1/2}(\theta_\gamma^\star - \widehat{\theta}_\gamma)}\lesssim \frac{\epsop}{\smin(I - \gamma \cov^{-1/2}\cross\cov^{-1/2})^2} \twonorm{\cov^{-1/2} \theta_{\phi,r}} + \frac{\epsz}{\smin(I - \gamma \cov^{-1/2}\cross\cov^{-1/2})}.
\end{align*}
\begin{lemma}[Theorem 3.8 in \cite{stewart1990matrix}]
\label{lemma:matrix_perturbation}
  Let $A \in \R^{m \times n}$, with $m \geq n$ and let $\widetilde{A} = A+E$. Then
  \begin{align*}
    \| \widetilde{A}^\dagger - A^\dagger \|_{\op} \leq \frac{1+\sqrt{5}}{2}\max\{ \|\widetilde{A}^\dagger\|_{\op}^2, \|A^\dagger\|_{\op}^2\} \|E\|_{\op}.
  \end{align*}
Furthermore, if $\opnorm{E} \leq \frac{1}{2} \smin(A)$, then 
\begin{align*}
	\opnorm{\widetilde{A}^\dagger - A^\dagger} \lesssim \opnorm{A^\dagger}^2 \opnorm{E}.
\end{align*}
\end{lemma}

\subsection{Proof of \Cref{prop:stability_invertibility}: Relating stability and invertibility}

The first part of the proposition follows directly from \Cref{fact:neumann_series}. For the second, again using \Cref{fact:neumann_series}:
\begin{align*}
	1 / \smin(I - \whitecross)  &= \opnorm{(I - \whitecross)^{-1}} \\
	& = \opnorm{\sum_{k=0}^\infty (\whitecross)^k} \\ 
	& \leq \sum_{k=0}^\infty \opnorm{(\whitecross)^k}. \\ 
	& \leq \sum_{k=0}^\infty \cond(P_\gamma)^{1/2} \left(1 - \frac{1}{\opnorm{P_\gamma}}\right)^{k/2}
\end{align*}
Here, we've used \Cref{lemma:stability_margin} in the last line. The final bound follows from applying the final argument from the proof of \Cref{theorem:main_result}. 

\subsection{Proof of \Cref{prop:wenlong}: Relationship to \cite{mou2020OptimalOI}}

The result follows from the proof of Corollary 1 in \cite{mou2020OptimalOI}. We include the calculation for the sake of completeness. For any unit vector $u$,
\begin{align*}
	(1 - \kappa) \twonorm{u}^2 \leq u^\top (I - \whitecross) u \leq \opnorm{(I -\whitecross)u} \twonorm{u}.
\end{align*}
Therefore, $\opnorm{(I - \whitecross)^{-1}} = 1/ \smin(I - \whitecross) \leq 1 / (1-\kappa)$.

\subsection{Proof of \Cref{prop:kolter}: contractivity implies stability}

By the Schur Complement Lemma, the contractivity condition implies that 
\begin{align*}
	\cov - \cross \cov^{-1} \cross^\top \succeq 0.
\end{align*}
Rearranging and multiplying on the left and the right by $\cov^{-1/2}$, 
\begin{align*}
	I \succeq (\cov^{-1/2} \cross \cov^{-1/2}) (\cov^{-1/2} \cross \cov^{-1/2})^\top.
\end{align*}
Using the fact that $\gamma \in (0,1)$ and the identity that for any matrix $A$, $\opnorm{A}^2 = \opnorm{AA^\top}$, we conclude
\begin{align*}
	\opnorm{\gamma \cov^{-1/2} \cross \cov^{-1/2}}^2 < 1. 
\end{align*} 
Stability follows from the observation that the spectral radius of a matrix is always smaller than the operator norm.

\subsection{Proof of \Cref{prop:counterexamples}: gaps between stability and other conditions}

Consider the following MDP with 4 states and no actions,  

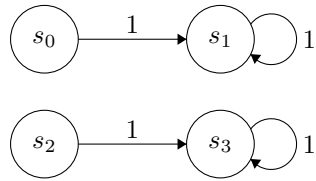
\begin{center}
\begin{tikzpicture}[scale=0.15]
\tikzstyle{every node}+=[inner sep=0pt]
\draw [black] (30.6,-27.5) circle (3);
\draw (30.6,-27.5) node {$s_0$};
\draw [black] (46.2,-27.5) circle (3);
\draw (46.2,-27.5) node {$s_1$};
\draw [black] (30.6,-36.8) circle (3);
\draw (30.6,-36.8) node {$s_2$};
\draw [black] (46.2,-36.8) circle (3);
\draw (46.2,-36.8) node {$s_3$};
\draw [black] (33.6,-27.5) -- (43.2,-27.5);
\fill [black] (43.2,-27.5) -- (42.4,-27) -- (42.4,-28);
\draw (38.4,-27) node [above] {$1$};
\draw [black] (33.6,-36.8) -- (43.2,-36.8);
\fill [black] (43.2,-36.8) -- (42.4,-36.3) -- (42.4,-37.3);
\draw (38.4,-36.3) node [above] {$1$};
\draw [black] (48.88,-35.477) arc (144:-144:2.25);
\draw (53.45,-36.8) node [right] {$1$};
\fill [black] (48.88,-38.12) -- (49.23,-39) -- (49.82,-38.19);
\draw [black] (48.88,-26.177) arc (144:-144:2.25);
\draw (53.45,-27.5) node [right] {$1$};
\fill [black] (48.88,-28.82) -- (49.23,-29.7) -- (49.82,-28.89);
\end{tikzpicture}
\end{center}

The reward distribution at every state is a mean-zero coin toss: $R(s) = \textrm{Unif}(\{\pm 1\})$ for all $s \in \states$. Now, consider the two-dimensional feature mapping,
\begin{align*}
	\phi(s_0) = [1,0]^\top,\; \phi(s_1) = [0, 1/\eps]^\top, \;\phi(s_2) = [0,1]^\top, \; \phi(s_3) = [\epsilon, 0]^\top,
\end{align*}
where $\epsilon> 0$ is a problem parameter to be determined later. This MDP is (trivially) linearly realizable with $\tsg = 0$ because all rewards have 0 mean. If $\cD$ place probability 1/2 on $s_0$ and $s_2$, then 
\begin{align*}
	\whitecross = \gamma \begin{bmatrix}
		0 & 1/\epsilon \\
		\epsilon & 0
	\end{bmatrix}.
\end{align*}
This matrix has eigenvalues equal to $\gamma$ and $-\gamma$ for all values of $\epsilon >0$. Hence, its spectral radius of this matrix is always strictly smaller than 1 and the OPE instance is stable (and hence invertible).

For this problem, we can check that 
\begin{align*}
	\nextcov = \frac{1}{2} \begin{bmatrix}
		\epsilon^2 & 0  \\
		0 & 1/\epsilon^2
	\end{bmatrix} \text{ and } \cov = \frac{1}{2} \begin{bmatrix}
		1 & 0 \\
		0 & 1
	\end{bmatrix}. 
\end{align*}
Therefore, $\cds$ is the smallest positive number $\beta$ such that 
\begin{align*}
0\preceq \frac{1}{2}
\begin{bmatrix}
		\beta - \epsilon^2 & 0 \\
		0 & \beta - 1/\epsilon^2
	\end{bmatrix} 
\end{align*}

\paragraph{Low distribution shift.} While stability holds for all values of $\epsilon >0$, as $\epsilon \rightarrow 0$, $\cds$ goes to $\infty$ (because $1/\epsilon^2$ becomes arbitrarily large). Hence, stability holds, but low distribution shift does not. This proves the first case. 

\paragraph{Symmetric stability.} Similarly, as $\epsilon \rightarrow 0$, we can check that the two eigenvalues of 
\begin{align*}
	\whitecross + (\whitecross)^\top,
\end{align*}
go to $\pm \infty$. Therefore, the symmetric stability condition (\Cref{ass:wenlong}) also fails for this problem. 

\paragraph{Contractivity.} From the argument in \Cref{prop:kolter}, we know that if contractivity (\Cref{ass:kolter}) held, then 
\begin{align*}
	\opnorm{\whitecross} < 1.
\end{align*}
However, a direct calculation shows that as $\epsilon \rightarrow 0 $, then $\opnorm{\whitecross} \rightarrow \infty$. Therefore, while stability holds, contractivity does not. 

\paragraph{Bellman completeness.} To prove the last case, we use a different example. In particular, consider the following MDP (with no actions) presented in \cite{amortila2020variant}, 

\begin{center}
\begin{tikzpicture}[scale=0.15]
\tikzstyle{every node}+=[inner sep=0pt]
\draw [black] (30.6,-27.5) circle (3);
\draw (30.6,-27.5) node {$s_0$};
\draw [black] (46.2,-27.5) circle (3);
\draw (46.2,-27.5) node {$s_1$};
\draw [black] (33.6,-27.5) -- (43.2,-27.5);
\fill [black] (43.2,-27.5) -- (42.4,-27) -- (42.4,-28);
\draw (38.4,-27) node [above] {$1$};
\draw [black] (48.88,-26.177) arc (144:-144:2.25);
\draw (53.45,-27.5) node [right] {$1$};
\fill [black] (48.88,-28.82) -- (49.23,-29.7) -- (49.82,-28.89);
\end{tikzpicture}
\end{center}
The rewards are $R(s_0) =0$ (almost surely) and $\E R(s_1) = 1$. The value function of any policy is linearly realizable in the feature mapping $\phi(s_0) = \gamma$ and $\phi(s_1) = 1$ with $\tsg = 1/ (1-\gamma)$. If the offline distribution places mass $1/2$ on each state then, 
\begin{align*}
	\whitecross = \left(\frac{\gamma}{2}\right) \frac{\gamma^2 + 1}{\gamma + 1}.
\end{align*}
This matrix (scalar) lies in the interval $(0,1)$ and is hence clearly stable and  invertible. However, Bellman completeness fails for this MDP. In particular, Bellman completeness asserts that for every $\theta$ there exist a $\theta'$ such that for all $s\in \states$, 
\begin{align*}
	\phi(s) \theta' = \E R(s,a) + \gamma \E_{s' \sim P(\cdot \mid s)} \phi(s') \cdot \theta
\end{align*}
In this case, this means that for all $\theta$, there exists a $\theta'$ such that 
\begin{align*}
	\phi(s_1) \cdot \theta' &= 1 + \gamma \cdot \phi(s_1) \cdot \theta \\ 
	\phi(s_0) \cdot \theta' &= 0 + \gamma \cdot \phi(s_1) \cdot \theta.
\end{align*}
Plugging in our choice of feature map, these equations become $\theta' = 1 + \gamma \cdot \theta$ and $\gamma \cdot \theta' = \gamma \cdot \theta$.
They clearly cannot be satisfied if we pick any $\theta \neq 0$.

\subsection{Bellman Residual Minimization Counterexample}
\label{subsec:brm_counterexample}

Consider the following 3 state MDP with no actions and stochastic transitions:

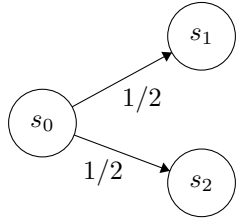
\begin{center}
\begin{tikzpicture}[scale=0.15]
\tikzstyle{every node}+=[inner sep=0pt]
\draw [black] (28.6,-29.8) circle (3);
\draw (28.6,-29.8) node {$s_0$};
\draw [black] (42.7,-35.1) circle (3);
\draw (42.7,-35.1) node {$s_2$};
\draw [black] (42.7,-22.1) circle (3);
\draw (42.7,-22.1) node {$s_1$};
\draw [black] (31.41,-30.86) -- (39.89,-34.04);
\fill [black] (39.89,-34.04) -- (39.32,-33.29) -- (38.97,-34.23);
\draw (33.98,-32.99) node [below] {$1/2$};
\draw [black] (31.23,-28.36) -- (40.07,-23.54);
\fill [black] (40.07,-23.54) -- (39.13,-23.48) -- (39.6,-24.36);
\draw (37.42,-26.45) node [below] {$1/2$};
\end{tikzpicture}
\end{center}
The feature mapping is:
\begin{align*}
	\phi(s_0) = \frac{\gamma}{4}, \; \phi(s_1) = \frac{1}{2},\; \phi(s_2) = 0.
\end{align*}
Rewards are exactly 0 everywhere except for $s_1$, where $r(s_1) = 1$ deterministically. We can check that this example is linearly realizable with $\tsg = \frac{1}{1-\gamma}$. However, it also holds that 
\begin{align*}
	\cov = \frac{\gamma^2}{16},\; \cross = \frac{\gamma}{16},\; \nextcov = \frac{1}{8}.
\end{align*}
Hence, $\cov - \gamma \cross - \gamma \cross^\top. + \gamma^2 \nextcov \succ 0$, but BRM returns the wrong answer,
\begin{align*}
\tbrm = (\cov - \gamma \cross - \gamma \cross^\top + \gamma^2 \nextcov)^\dagger (\tpr - \gamma \E \phi(s', a') r(s, a)) = 0,
\end{align*}
since $\E \phi(s,a) r(s,a) = \E \phi(s',a')r(s,a) = 0$.

\subsection{Proof of \Cref{theorem:lower_bound_lspe}: necessity of invertibility for LSTD}

We begin by proving two auxiliary claims and then move on to proving each part of the theorem separately.

\begin{claim}
\label{claim:rank_deficient}
If the matrix $I-\gamma \cov^{-1}\cross$ is singular, then there exists a real vector $v \in \R^d$ such that: 
\begin{align*}
	\underset{(s,a) \sim \calD, s' \sim P(\cdot |s,a), a' \sim \pi(s')}{\E} \phi(s,a) \langle \gamma \cdot \phi(s',a') - \phi(s,a), v \rangle=0.
\end{align*}
\end{claim}
\begin{proof}
The matrix being rank deficient implies that there exists a vector $v$ such that $(I-\gamma \cov^{-1} \cross)v = 0$, or equivalently, that the matrix $\gamma\cov^{-1}\cross$ has an eigenvector $v$ with eigenvalue 1. Because the matrix and eigenvalue are both real, we can also take $v$ to be real. From here, $v = \gamma \cov^{-1} \cross v$. Hence, $\cov v = \gamma \cross v$. Expanding out the definitions of these matrices,
\begin{align*}
	\E \phi(s,a) \langle \phi(s,a), v \rangle = \gamma \E \phi(s,a) \langle \phi(s', a'), v \rangle.
\end{align*}
Rearranging both terms to be on the same side we get the claim.
\end{proof}

\begin{claim}
\label{claim:telescoping}
For any $(s,a) \in \states \times \actions$, 
\begin{align*}
	\phi(s,a) = -\E \left[\sum_{t=0}^\infty \gamma^t\left( \gamma\phi(s_{t+1}, a_{t+1}) - \phi(s_t, a_t) \right) \mid (s_0, a_0) = (s,a), \pi \right].
\end{align*}
\end{claim}
\begin{proof}
The sum telescopes and $\lim_{t \rightarrow \infty} \gamma^t \E \phi(s_t, a_t) = 0$.
\end{proof}

We conclude with the proof of \Cref{theorem:lower_bound_lspe}:
\paragraph{Alternate reward.} As per the presentation of theorem, the only difference between $\calM$ and $\altm$ is the unknown reward. In particular, we define the new reward function $\altreward$ as 
\begin{align}
\label{eq:new_reward}
	\altreward(s,a) = R(s,a) + \frac{1}{2B}\langle \gamma \cdot \phi(s',a') - \phi(s,a), v \rangle
\end{align}
where $v$ is as in \Cref{claim:rank_deficient}, $s' \sim P(\cdot \mid s,a)$, $a' \sim \pi(s')$, and $B = \sup_{s,a} \twonorm{\phi(s,a)}$. Note that by Cauchy-Schwarz, and the definition of B, for any $s,s',a,$ and $a'$:
\begin{align*}
|\frac{1}{2B}\langle \gamma \cdot \phi(s',a') - \phi(s,a), v \rangle|  \leq \frac{1}{2B} \twonorm{v}(\twonorm{\phi(s,a)} + \twonorm{\phi(s',a')})\leq 1
\end{align*}
Therefore, $|r(s,a)|$ is uniformly bounded by 2.

\paragraph{Proof of identical moments.} Since the features, offline distribution, and transitions are all the same, then $\cov=\altcov, \cross = \altcross$, and $\nextcov = \altnext$. Next, by expanding out the new reward function:
\begin{align*}
	\altthetar &= \E_{(s,a) \sim \calD} \phi(s,a) \bar{r}(s,a) \\ 
		&=  \E \phi(s,a) r(s,a) + \frac{1}{2B}\underset{(s,a) \sim \calD, s' \sim P(\cdot |s,a), a' \sim \pi(s')}{\E} \phi(s,a) \langle \gamma \cdot \phi(s',a') - \phi(s,a),  v \rangle \\ 
		& =  \E \phi(s,a) r(s,a) + 0,
\end{align*}
where the last line follows from  \Cref{claim:rank_deficient}.


\paragraph{Proof of realizability.} Expanding out the definition of $\altqpi$, 

\begin{align*}
	\altqpi(s,a) &= \E \left[ \sum_{t=0}^\infty \gamma^t \cdot \bar{r}(s_t, a_t)   \mid (s_0,a_0)=(s,a), \pi\right] \\ 
	& = \E \left[ \sum_{t=0}^\infty \gamma^t \cdot r(s_t, a_t)   \mid (s_0,a_0)=(s,a), \pi\right] \\\
	&+ \E \left[ \sum_{t=0}^\infty \gamma^t \cdot \langle \gamma \cdot \phi(s_{t+1}, a_{t+1}) - \phi(s_t, a_t),  \frac{1}{2B} v \rangle  \mid (s_0,a_0)=(s,a), \pi\right] \\ 
	& = Q^\pi(s,a) - \phi(s,a)^\top v \frac{1}{2B} \\
	& =\phi(s,a)^\top(\theta_\gamma^\star - \frac{1}{2B}v),
\end{align*}
where in the 3rd line we have used \Cref{claim:telescoping} and in the last one used the assumption that $\Qpi$ is linearly realizable. In short, $\altqpi$ is linearly realizable with weight vector $\theta^\star_\gamma - (2B)^{-1}v$.

\paragraph{Proof of different Q functions} By the previous part establishing the realizability of $\altqpi$,
\begin{align*}
	\E_{\calD} (\Qpi(s,a) - Q'^\pi(s,a))^2 = \frac{1}{4B^2}v^\top \cov v \geq \frac{\smin(\cov)}{4B^2} \twonorm{v}^2.
\end{align*}
The precise statement follows from the fact that $v$ has unit length.

\section{Concentration Analysis: Proof of \Cref{lemma:main_concentration_result}}

\begin{lemma}[Matrix Bernstein, \cite{tropp2012user}]
\label{lemma:matrix_bernstein}
  Let $S_1,\ldots,S_n \in \R^{d_1 \times d_2}$ be  random, independent matrices satisfying $\E[S_k] = 0$, $\max\{\|\E[S_kS_k^\top]\|_{\op}, \|\E[S_k^\top S_k]\|_{\op} \} \leq \sigma^2$, and $\opnorm{S_k} \leq L$ almost surely for all $k$. 
Then, with probability at least $1-\delta$ for any $\delta \in (0,1)$,
  \begin{align*}
    \|\frac{1}{n} \sum_{k=1}^n S_k\|_{\op} \leq \sqrt{ \frac{2\sigma^2\log( (d_1+d_2)/\delta)}{n}} + \frac{2L\log( (d_1+d_2)/\delta)}{3n}.
  \end{align*}
\end{lemma}

\begin{lemma}[Vector Bernstein,  \cite{minsker2017some}]
\label{lemma:vector_bernstein}
Let $v_1, \dots, v_n$ be independent vectors in $\R^d$ such that $\E v_k =0$, $\E \twonorm{v_k}^2 \leq \sigma^2$, and $\twonorm{v_k} \leq L$ almost surely for all $k$. Then, with probability $1-\delta$ for any $\delta \in (0,1)$,
\begin{align*}
	\twonorm{\frac{1}{n} \sum_{i=1}^n v_i} \leq \sqrt{\frac{2\sigma^2 \log(28 / \delta)}{n}} + \frac{2 L \log(28 /\delta)}{3n}.
\end{align*} 
\end{lemma}
To shorten the notation in our concentration analysis, we use $x_i = \phi(s_i,a_i)$ and $y_i=\phi(s'_i, a'_i)$, and $r_i = r(s_i,a_i)$. With this shorthand: 
\begin{align}
\label{eq:concentration_shorthand}
	\cov = \E xx^\top, \quad \covhat = \frac{1}{n}\sum_{i=1}^n x_ix_i^\top,\quad 
	\cross = \E xy^\top, \quad \crosshat =\frac{1}{n} \sum_{i=1}^n x_i y_i^\top,
\end{align}
\begin{align*}
\quad  \theta_{\phi,r} = \E xr,\quad \widehat{\theta}_{\phi,r} = \frac{1}{n} \sum_{i=1}^n x_i r_i.
\end{align*}
\subsection{Bounding $\epsop$}

\begin{lemma}
\label{lemma:epsop_concentration}
If $n \gtrsim \levcov^2 \log(d/\delta)$ then, with probability $1-\delta$, 
\begin{align*}
\| \cov^{1/2}(\gamma \covhat^{-1}\crosshat)\cov^{-1/2} - \gamma \cov^{-1/2}\cross\cov^{-1/2} \|_\op &\lesssim \sqrt{ \frac{ \max(\crossvar, \varcov \Cds) \log(d/\delta)}{n}} 
+ \frac{\max(\Cds^{1/2} \levcov^2, \levcov \levnext) \log(d/\delta)}{n}.
\end{align*}
\end{lemma}

\begin{proof}
Let $\Ahat \defeq \gamma \covhat^{-1} \crosshat$. We start by using the following error decomposition,
\begin{align*}
& \| \cov^{1/2}\Ahat\cov^{-1/2} - \gamma \cov^{-1/2}\cross\cov^{-1/2} \|_\op\\
& \leq \gamma \| \cov^{1/2}\covhat^{-1}\cov^{1/2}\cdot \cov^{-1/2}\left(\crosshat - \cross\right)\cov^{-1/2}\|_{\op}
  + \gamma \| \cov^{1/2} (\covhat^{-1} - \cov^{-1})\cov^{1/2} \cdot \cov^{-1/2} \cross \cov^{-1/2} \|_{\op}\\
& \leq \gamma \underbrace{\| \cov^{1/2}\covhat^{-1}\cov^{1/2}\|_{\op}}_{\defeq T_1} \cdot \underbrace{\|\cov^{-1/2}\left(\crosshat - \cross\right)\cov^{-1/2}\|_{\op}}_{\defeq T_2}\\
  &+  \underbrace{\| \cov^{1/2} (\covhat^{-1} - \cov^{-1})\cov^{1/2}\|_{\op}}_{\defeq T_3} \cdot \underbrace{\| \gamma \cov^{-1/2} \cross \cov^{-1/2} \|_{\op}}_{\defeq T_4}.
\end{align*}
We now bound each of these terms separately.

\paragraph{Bound on $T_2$.}
We apply the Matrix Bernstein inequality on
$\cov^{-1/2}\left(\crosshat - \cross\right)\cov^{-1/2}$. Here we define
\begin{align*}
S_k = \cov^{-1/2}\left(x_ky_k^\top - \cross\right)\cov^{-1/2}
\end{align*}
which is centered and satisfies:
\begin{align*}
  \| S_k \| & \leq \| \cov^{-1/2} x_ky_k^\top \cov^{-1/2}\| + \E_{\calD} \| \cov^{-1/2} xy^\top \cov^{-1/2}\| \leq 2\sup_{(x,y) \in \textrm{supp}(\calD)} \| \cov^{-1/2} x y^\top \cov^{-1/2} \|\\
  & \leq 2 \sup_{(x,y) \in \textrm{supp}(\calD)} \| \cov^{-1/2} x \| \cdot \| \cov^{-1/2} y\|  \leq 2 \levcov \levnext.
\end{align*}
Therefore for $\crossvar$ defined as in \Cref{eq:cross_var_def}, , we get that with probability $1-\delta$,
\begin{align*}
T_2 \leq \sqrt{ \frac{2\crossvar \log(2d/\delta)}{n}} + \frac{ 4\levcov \levnext \log(2d/\delta)}{3n}.
\end{align*}

\paragraph{Bound on $T_1$ and $T_3$.}
Essentially the same argument as for the bound on $T_2$ reveals that,
\begin{align}
\label{eq:tau_def}
\| \cov^{-1/2}(\covhat - \cov)\cov^{-1/2}\|_{\op} \leq \sqrt{ \frac{2\varcov \log(2d/\delta)}{n}} + \frac{ 2\levcov^2 \log(2d/\delta)}{3n} =: \tau.
\end{align}
This inequality directly implies that
\begin{align*}
1 - \tau \leq \lambda_{\min}(\cov^{-1/2}\covhat\cov^{-1/2}) \leq \lambda_{\max}(\cov^{-1/2}\covhat\cov^{-1/2}) \leq 1+\tau,
\end{align*}
which in particular implies that $\cov^{-1/2}\covhat\cov^{-1/2}$ is
invertible whenever $\tau < 1/2$, a fact that is ensured by our lower bound on $n$. Therefore:
\begin{align}
\label{eq:whitened_covinv}
  T_1 = \| \cov^{1/2}\covhat^{-1}\cov^{1/2} \| = \frac{1}{\lambda_{\min}( \cov^{-1/2}\covhat\cov^{-1/2} )} \leq \frac{1}{1-\tau}.
\end{align}
More generally, we have that:
\begin{align*}
  1-2\tau \leq \lambda_{\min}(\cov^{1/2}\covhat^{-1}\cov^{1/2}) \leq \lambda_{\max}(\cov^{1/2}\covhat^{-1}\cov^{1/2}) \leq 1+2\tau.
\end{align*}
Using the fact that $1/(1+\tau) \geq 1-2\tau$ and $1/(1-\tau) \leq
1+2\tau$ for $\tau \leq 1/2$, this directly yields
\begin{align}
\label{eq:covinv_bound}
 T_3 = \| \cov^{1/2}(\covhat^{-1} - \cov^{-1})\cov^{1/2}\| \leq 2\tau.
\end{align}
Thus, we have bounded $T_1$ and $T_3$. In particular, for $\tau < 1/2$, $T_1 \leq 2$, and $T_3 \leq 2\tau$.
\paragraph{Bound on $T_4$.}
For $T_4$, no concentration argument is required. Instead, a Schur complement argument implies that,
\begin{align*}
\| \cov^{-1/2} \cross \cov^{-1/2} \|_{\op}^2 \leq \| \cov^{-1/2}\nextcov \cov^{-1/2}\|_{\op}\leq \Cds,
\end{align*}
where we've used $\nextcov \preceq \cds \cov$. Hence, $T_4 \leq \Cds^{1/2}$.
\paragraph{Wrapping up.}
Taking a union bound, we obtain that
\begin{align*}
\varepsilon_{\op} \lesssim \sqrt{ \frac{ \max(\crossvar, \varcov \Cds) \log(d/\delta)}{n}} + \frac{\max(\Cds^{1/2} \levcov^2, \levcov \levnext) \log(d/\delta)}{n}.
\end{align*}
\end{proof}

\subsection{Bounding $\epsz$}

\begin{lemma}
\label{lemma:epsz_concentration}
If $n \gtrsim \levcov^2 \log(d/\delta)$ then, with probability $1-\delta$, 
\begin{align*}
\twonorm{\cov^{1/2} \covhat^{-1} \widehat{\theta}_{\phi,r}  - \cov^{-1/2} \theta_{\phi,r}} \lesssim \sqrt{ \frac{  \max(\twonorm{ \cov^{-1/2}\tpr}^2\varcov, \rewardvar) \log(d/\delta)}{n}} + \frac{ \twonorm{\cov^{-1/2} \theta_{\phi,r}} \levcov^2 \log(d/\delta)}{n}.
\end{align*}
\end{lemma}

\begin{proof}
The ideas are very similar to  \Cref{lemma:epsop_concentration}. In this case, the relevant error decomposition is, 
\begin{align*}
	\epsz &= \twonorm{\cov^{1/2} \covhat^{-1} \widehat{\theta}_{\phi,r}  - \cov^{-1/2} \theta_{\phi,r}} \\
		  &\leq \underbrace{\opnorm{\cov^{1/2} \covhat^{-1}\cov^{1/2}}}_{\defeq T_1} \underbrace{\twonorm{\cov^{-1/2}(\theta_{\phi,r} - \widehat{\theta}_{\phi,r})}}_{\defeq T_2} + \underbrace{\opnorm{(\cov^{1/2} \covhat^{-1} \cov^{1/2} - I)}}_{\defeq T_3} \twonorm{\cov^{-1/2}\theta_{\phi,r}}.
\end{align*}
\paragraph{Bound on $T_1$ and $T_3$.} 
Whenever $\tau$, defined as in \Cref{eq:tau_def}, is strictly less than $1/2$, the analysis therein (in particular, \Cref{eq:covinv_bound} and \Cref{eq:whitened_covinv}) proves that $T_1 \leq 2$ and $T_3 \leq 2\tau$. 

\paragraph{Bound on $T_2$.}
We apply the vector Bernstein inequality, \Cref{lemma:vector_bernstein}, on the vectors 
\begin{align*}
v_i = \cov^{-1/2}x_i r_i - \cov^{-1/2} \theta_{\phi,r}.
\end{align*}
 Note that, since the rewards have magnitude bounded by 1,
\begin{align*}
	\sup_i \twonorm{v_i} \leq \sup_i \twonorm{\cov^{-1/2} x_i r_i} + \twonorm{\cov^{-1/2} \theta_{\phi,r}} \leq \twonorm{\cov^{-1/2} x_i r_i} + \E \twonorm{\cov^{-1/2}  xr}  \leq 2\levcov,
\end{align*}
and, 
\begin{align*}
	\E \twonorm{v_i}^2 = \E \twonorm{\cov^{-1/2} x_i r_i}^2 - \twonorm{\cov^{1/2} \theta_{\phi,r}}^2= \rewardvar .
\end{align*}
Applying vector Bernstein, 
\begin{align*}
	T_2 \leq  \sqrt{\frac{2 \rewardvar \log(28 / \delta)}{n}} + \frac{4 \levcov \log(28 /\delta)}{3n}.
\end{align*}
\paragraph{Wrapping up.} Combining these, we get that, 
\begin{align*}
	\epsz \lesssim \sqrt{ \frac{  \max(\twonorm{ \cov^{-1/2}\tpr}^2\varcov, \rewardvar
	) \log(d/\delta)}{n}} + \frac{ \twonorm{\cov^{-1/2}\theta_{\phi,r}} \levcov^2 \log(d/\delta)}{n}.
\end{align*}
\end{proof}

\subsection{Bounding variances}
\label{subsec:variance_bounds}

\paragraph{Bounding $\rewardvar$} Since the rewards $r(s,a)$ satisfy $|r(s,a)| \leq 1$, we have that
\begin{align*}
	\E \twonorm{v_i}^2 = \E \twonorm{\cov^{-1/2} x_i r_i}^2 - \twonorm{\cov^{1/2} \theta_{\phi,r}}^2 \leq \traceb{\cov^{-1/2} \E r_i^2 x_i x_i^\top \cov^{-1/2}} \leq d.
\end{align*}

\paragraph{Bounding $\crossvar$.}Again using the notation from \Cref{eq:concentration_shorthand}, and letting 
\begin{align*}
S_k = \cov^{-1/2}\left(x_ky_k^\top - \cross\right)\cov^{-1/2}
\end{align*}
bounding $\crossvar$ is equivalent to bounding the operator norms of:
\begin{align*}
\E[S_kS_k^\top] &= \E[ \|\cov^{-1/2}y\|_2^2 (\cov^{-1/2} x)(\cov^{-1/2} x)^\top] - \cov^{-1/2}\cross\cross^\top\cov^{-1/2}\\
\E[S_k^\top S_k] &= \E[ \|\cov^{-1/2}x\|_2^2 (\cov^{-1/2} y)(\cov^{-1/2} y)^\top] - \cov^{-1/2}\cross^\top \cross\cov^{-1/2}.
\end{align*}
We will subsequently show that, for any vector $v \in
\R^d$, we have
\begin{align}
  v^\top (\E[S_kS_k^\top])v \geq 0, \quad v^\top (\E[S_k^\top S_k])v \geq 0. \label{eq:positive_qf}
\end{align}
Additionally, for any random variables $(a,b) \in \R\times\R^d$ from some joint distribution, Holder's inequality implies that
\begin{align*}
  \opnorm{\E[a^2 bb^\top]} = \sup_{v, \twonorm{v} = 1} \E[a^2 (v^\top b)^2] &\leq \min\{ \sup\{a\}\sup_{v} \E[ (v^\top b)^2], \sup_{b,v}\{(v^\top b)^2\} \E[a^2]\}\\
  &= \min\{ \sup\{a\} \opnorm{ \E[bb^\top]}, \sup\{\|b\|_2^2\}\E[a^2]\}.
\end{align*}
Using these two facts and positive semi-definiteness, we have that
\begin{align*}
\opnorm{\E[S_kS_k^\top]} &\leq \opnorm{ \E[ \|\cov^{-1/2}y\|_2^2 (\cov^{-1/2} x)(\cov^{-1/2} x)^\top] } \leq \sup_{y} \| \cov^{-1/2}y\|_2^2 \| \E[ (\cov^{-1/2} x)(\cov^{-1/2} x)^\top] \| \leq  \levnext^2 .
\end{align*}
Essentially the same proof yields a similar bound on $\| \E[S_k^\top S_k]\|$:
\begin{align*}
\opnorm{ \E[S_k^\top S_k] } &\leq \opnorm{ \E[ \|\cov^{-1/2}x\|_2^2 (\cov^{-1/2} y)(\cov^{-1/2} y)^\top] } \leq \rho_0^2 \opnorm{ \cov^{-1/2}\nextcov\cov^{-1/2} } \leq  \levcov^2 \Cds.
\end{align*}
Alternatively, we can get
\begin{align*}
\opnorm{ \E[S_k^\top S_k] } &\leq \opnorm{ \E[ \|\cov^{-1/2}x\|_2^2 (\cov^{-1/2} y)(\cov^{-1/2} y)^\top] } \leq \E[ \|\cov^{-1/2}x\|_2^2 \| (\cov^{-1/2} y)(\cov^{-1/2} y)^\top] \| ]\\
  & =  \E[ \|\cov^{-1/2}x\|_2^2 \|\cov^{-1/2} y\|_2^2 ] \leq \levnext^2 d.
\end{align*}
Let us now verify~\eqref{eq:positive_qf}. Rebinding $\tilde{x} =
\cov^{-1/2}x, \tilde{y} = \cov^{-1/2}y$, we have
\begin{align*}
  v^\top (\E[S_kS_k^\top])v = \E[ (v^\top \tilde{x})^2 \| \tilde{y} \|_2^2] - (\E(v^\top\tilde{x})\tilde{y})^\top (\E(v^\top\tilde{x})\tilde{y}) = \E \| (v^\top\tilde{x})\tilde{y}\|_2^2 - \| \E[ (v^\top\tilde{x})\tilde{y}] \|_2^2 \geq 0,
\end{align*}
where the last inequality is by convexity. In conclusion,
\begin{align*}
\crossvar \leq \max(\levnext^2, \min(\levcov^2\Cds, \levnext^2 d)).
\end{align*}

\paragraph{Bounding $\varcov$.}

For $\tilde{x} =\cov^{-1/2} \phi(s,a)$, the variance $\varcov$ is equal to $$\varcov = \opnorm{\E \tilde{x} \tilde{x}^\top \tilde{x} \tilde{x}^\top - I} = \opnorm{\E \twonorm{\tilde{x}}^2 \tilde{x}\tilde{x}^\top - I}.$$ 
While this quantity is always less that $\levcov^2$, one can achieve tighter bounds if the offline distribution is \emph{hypercontractive} as per the following definition:

\begin{definition}
A distribution $\calD$ over random vectors $x$ is $L8$-$L2$ hypercontractive if there exists a positive constant $L$ such that for all unit vectors $u$,
\begin{align*}
	\E_{x \sim \calD} ((x - \E x)^\top u)^8 \leq L^2 
	\left( \E_{x \sim \calD} ((x-\E x)^\top u)^2 \right)^4.
\end{align*}
\end{definition}
Gaussians or strongly log-concave distributions are some examples of probability measures that satisfy this condition. If $\cov^{-1/2}\phi(s,a)$ is $L8$-$L2$ hypercontractive, then one can show that 
\begin{align*}
\varcov\lesssim L \traceb{I +  \mu \mu^\top} \opnorm{I + \mu \mu^\top},
\end{align*}
where $\mu \defeq  \cov^{-1/2}\E_{(s,a) \sim \calD}\phi(s,a)$. We point the interested reader to Lemma A.3 in \cite{yesh} for a more formal derivation.

\newpage
\section{Analyzing the Misspecified Case: Proof of \Cref{prop:misspecification}}
By definition of $\tsi$, we have that for all $(s,a) \in \states \times \actions$ we can write $\Qpi$ as 
\begin{align}
\label{eq:qpi_decomp}
	\Qpi(s,a) = \phi(s,a)^\top \tsg + f(s,a),
\end{align}
where $f(s,a) = \Qpi(s,a) - \phi(s,a)^\top \tsg$ and $\sup_{s,a} |f(s,a)| \leq \epsinf$. 

From the relationship above, we have that for $\Qhat(s,a) = \phi(s,a)^\top \thetahat$
\begin{align}
	|\Qpi(s,a) - \Qhat(s,a)| &= |\phi(s,a)^\top(\tsi - \thetahat) + f(s,a)| \notag\\
	&\leq \twonorm{\cov^{-1/2} \phi(s,a)} \twonorm{\cov^{1/2}(\tsi - \thetahat)} + |f(s,a)| \label{eq:ms1}.
\end{align}
Applying the triangle inequality again, 
\begin{align}
\label{eq:ms2}
\twonorm{\cov^{1/2}(\tsi - \thetahat)} \leq \twonorm{\cov^{1/2}(\tsi - \tfp)} + \twonorm{\cov^{1/2}(\tfp - \thetahat)}.
\end{align}
By assumption on $\thetahat$, $\twonorm{\cov^{1/2}(\tfp - \thetahat)} \leq \epsfp$. Therefore, it remains to bound $\twonorm{\cov^{1/2}(\tsi - \tfp)}$. By \Cref{claim:misspecification}, we have that
\begin{align*}
	\cov^{1/2}\tsi &= (I -\whitecross)^{-1} \cov^{-1/2} \tpr \\
	&+(I -\whitecross)^{-1}  \cov^{-1/2} \underset{\substack{(s,a) \sim \calD\\s' \sim P(\cdot |s,a), a' \sim \pi(s')}}{\E}  \phi(s,a)(\phi(s,a)^\top \tsi -\gamma \phi(s', a')^\top \tsi -r(s,a)).
\end{align*}
Note that $\cov^{1/2}\tfp$ is exactly equal to $(I -\whitecross)^{-1} \cov^{-1/2} \tpr$. Furthermore, by the second part of \Cref{claim:misspecification}, the $\ell_2$ norm of the second term in the expression above is upper bounded by $\levcov \epsinf / \smin(I -\whitecross)$. Consequently,
\begin{align}
\label{eq:ms3}
	\twonorm{\cov^{1/2}(\tsi - \tfp)} \leq \frac{\levcov}{\smin(I -\whitecross)}  \epsinf.
\end{align}
Combining \Cref{eq:ms1,eq:ms2,eq:ms3}, we get that 
\begin{align*}
|\Qpi(s,a) - \Qhat(s,a)| \leq \twonorm{\cov^{-1/2} \phi(s,a)}(\epsfp+ \frac{\levcov}{\smin(I -\whitecross)}  \epsinf) + \epsinf.
\end{align*}
\newpage
\begin{claim}
\label{claim:misspecification}
Let $\tsi$ be defined as in \Cref{eq:tsi_def} and let $A \defeq I -\whitecross$ then, 
\begin{align*}
	\cov^{1/2}\tsi - \cov^{1/2}\theta_\star &= A^{-1} \cov^{-1/2} \tpr \\
	&+A^{-1}  \cov^{-1/2} \underset{\substack{(s,a) \sim \calD\\s' \sim P(\cdot |s,a), a' \sim \pi(s')}}{\E}  \phi(s,a)(\phi(s,a)^\top \tsi -\gamma \phi(s', a')^\top \tsi -r(s,a)).
\end{align*}
where 
\begin{align*}
	\twonorm{A^{-1}  \cov^{-1/2} \underset{\substack{(s,a) \sim \calD\\s' \sim P(\cdot |s,a), a' \sim \pi(s')}}{\E}  \phi(s,a)(\phi(s,a)^\top \tsi -\gamma \phi(s', a')^\top \tsi -r(s,a))} \leq \frac{\epsinf}{\smin(I - \whitecross)} \levcov.
\end{align*}
\end{claim}
\begin{proof}
By the Bellman equation, we have that,
\begin{align*}
	\Qpi(s,a) = \E r(s,a) + \gamma \cdot \underset{\substack{s' \sim P(\cdot\mid s,a) \\ a' \sim \pi(s')}}{\E} \Qpi(s', a').
\end{align*}
Using the decomposition from \Cref{eq:qpi_decomp}, the following relationship holds for all $(s,a) \in \states \times \actions$,
\begin{align*}
	\phi(s,a)^\top \tsg = \E r(s,a) + \gamma \cdot \underset{\substack{s' \sim P(\cdot\mid s,a) \\ a' \sim \pi(s')}}{\E} \phi(s', a')^\top \tsg - f(s,a) + \gamma \cdot \underset{\substack{s' \sim P(\cdot\mid s,a) \\ a' \sim \pi(s')}}{\E}f(s', a').
\end{align*} 
Now we do a couple of things, we multiply on the left by $
\cov^{-1/2}\phi(s,a)$ and take expectations with respect to $(s,a) \sim \cD$. Rearranging, we get the following equation: 
\begin{align*}
	\cov^{1/2} \tsg  &= (I - \gamma \cov^{-1/2} \cross \cov^{-1/2})^{-1}\cov^{-1/2} \tpr  \\
	& + (I - \gamma \cov^{-1/2} \cross \cov^{-1/2})^{-1} \underset{\substack{s' \sim P(\cdot\mid s,a) \\ a' \sim \pi(s')}}{\E} \cov^{-1/2} \phi(s,a)  (\gamma \cdot f(s', a') - f(s,a)) 
\end{align*}
Focusing on the second term, we have that for any $(s,a) \in \states \times \actions$, $|f(s,a)| \leq \epsinf$ and $\twonorm{\cov^{-1/2} \phi(s,a)}\leq \levcov$. Therefore, 
\begin{align*}
\twonorm{(I - \gamma \cov^{-1/2} \cross \cov^{-1/2})^{-1} \underset{\substack{s' \sim P(\cdot\mid s,a) \\ a' \sim \pi(s')}}{\E} \cov^{-1/2} \phi(s,a)  (\gamma \cdot f(s', a') - f(s,a))} \leq \frac{\epsinf \cdot \levcov}{\smin(I - \whitecross)}.
\end{align*}
Moreover, $f(s,a) = \Qpi(s',a') - \phi(s,a)^\top\tsi$ and $\Qpi(s,a) = \E r(s,a) + \gamma \cdot \underset{\substack{s' \sim P(\cdot\mid s,a) \\ a' \sim \pi(s')}}{\E} \Qpi(s', a')$.\\
Using these identities, we have that:
\begin{align*}
\gamma \cdot f(s', a') - f(s,a) = \phi(s,a)^\top \tsi - \gamma \phi(s', a')^\top \tsi - r(s,a).
\end{align*}
\end{proof}

\newpage

\end{document}